\documentclass[twoside,11pt]{article}
\usepackage{jmlr2e}
\usepackage{lmodern}
\usepackage[T1]{fontenc}
\usepackage{microtype}
\usepackage{csquotes}
\usepackage{float}
\usepackage{caption}
\usepackage{subcaption}
\usepackage{amsmath}
\usepackage{subdepth}
\usepackage{accents}
\usepackage{algorithm}
\usepackage{algorithmicx}
\usepackage[noend]{algpseudocode}
\usepackage{tikz}
\usetikzlibrary{positioning, arrows.meta}
\tikzset{%
  dot/.style n args = {4}{name=#3, circle, draw, inner sep=1pt, minimum size=4pt, fill=black, label={[shift={(#1,#2)}]#4:$#3$}},
  lat/.style n args = {4}{name=#3, circle, draw, inner sep=1pt, minimum size=4pt, label={[shift={(#1,#2)}]#4:$#3$}},
  >={Latex[width=1.5mm,length=2mm]},
  every picture/.style={semithick}
}

\newcommand\thickbar[1]{\accentset{\rule{.4em}{.8pt}}{#1}}
\newcommand\widebar[1]{\accentset{\rule{.8em}{.8pt}}{#1}}
\newcommand\thickubar[1]{\underaccent{\,\rule{.35em}{.8pt}}{#1}}

\def\independenT#1#2{\mathrel{\rlap{$#1#2$}\mkern2mu{#1#2}}}
\newcommand{\+}[1]{\ensuremath{\mathbf{#1}}}
\def\independenT#1#2{\mathrel{\rlap{$#1#2$}\mkern2mu{#1#2}}}

\algnewcommand\algorithmicinput{\textbf{INPUT:}}
\algnewcommand\INPUT{\item[\algorithmicinput]}
\algnewcommand\algorithmicoutput{\textbf{OUTPUT:}}
\algnewcommand\OUTPUT{\item[\algorithmicoutput]}

\newcommand{\doo}{\textrm{do}}

\newcommand{\Pa}{\textrm{Pa}}
\newcommand{\Ch}{\textrm{Ch}}
\newcommand{\An}{\textrm{An}}
\newcommand{\De}{\textrm{De}}
\newcommand{\Co}{\textrm{Co}}
\newcommand{\indep}{\protect\mathpalette{\protect\independenT}{\perp}}

\jmlrheading{18}{2018}{1-23}{9/17; Revised 5/18}{5/18}{17-563}{Santtu Tikka and Juha Karvanen}

\ShortHeadings{Enhancing Identification}{Tikka and Karvanen}
\firstpageno{1}

\begin{document}

\title{Enhancing Identification of Causal Effects by Pruning}

\author{\name Santtu Tikka \email santtu.tikka@jyu.fi \\
        \name Juha Karvanen \email juha.t.karvanen@jyu.fi \\
        \addr Department of Mathematics and Statistics \\
        P.O.Box 35 (MaD) FI-40014 University of Jyvaskyla, Finland}

\editor{Peter Spirtes}

\maketitle

\begin{abstract}%
Causal models communicate our assumptions about causes and effects in real-world phenomena. Often the interest lies in the identification of the effect of an action which means deriving an expression from the observed probability distribution for the interventional distribution resulting from the action. In many cases an identifiability algorithm may return a complicated expression that contains variables that are in fact unnecessary. In practice this can lead to additional computational burden and increased bias or inefficiency of estimates when dealing with measurement error or missing data. We present graphical criteria to detect variables which are redundant in identifying causal effects. We also provide an improved version of a well-known identifiability algorithm that implements these criteria.
\end{abstract}


\section{Introduction} \label{sect:intro}

A formal framework for causal inference is provided by the probabilistic causal model \citep{pearl09} that encodes our knowledge of the variables of interest and their mutual relationships. In observational studies experimentation is not available, but through the causal model framework we can still symbolically intervene on variables, forcing them to take certain values as if an experiment had taken place. The question is whether we can make inferences about the effect of the intervention in the post-intervention model using only the observed probability distribution of the variables in the model before the intervention took place. This question is formally defined as identifiability of causal effects, and it has received considerable attention in literature, including a number of algorithmic solutions \citep{huang06, shpitser06, tian02}.

A causal model can be associated with a directed acyclic graph (DAG) that represents the functional relationships of the variables included in the model. The graphical representation provides us with the concept of d-separation \citep{geiger90}, that can be used to infer conditional independences between variables from the graph. If the distribution of the variables implies no conditional independence statements other than those already encoded in the graph, we say that the distribution is faithful \citep{spirtes00}.

The use of d-separation in the post-intervention model is the basis of do-calculus \citep{pearl95}, which consists of a set of inference rules for manipulating interventional distributions. The purpose of do-calculus is to derive formulas for causal effects and other causal queries, and it has been shown to be complete with respect to the identifiability of causal effects \citep{huang06, shpitser06}. The derived formulas provide recipes for estimating the causal effects from observational data.

When computing causal effect formulas, we often apply an identifiability algorithm, such as the ID algorithm by \citet{shpitser06}. Criteria for identifiability such as the back-door criterion and front-door criterion are available for manual derivations \citep{pearl09} but the ID algorithm is more general and thus more suitable for automated processing. The ID algorithm splits the original problem into smaller subproblems which are then solved and aggregated as the final expression for the causal effect. 

Complicated expressions are likely to arise in situations where we have included variables in our model that do not provide further benefit for the identification of the causal effect of interest. It is often the case that these variables nonetheless appear in the resulting formula, and deriving a simpler expression with the variable eliminated can be non-trivial. It is hard to specify what makes one expression simpler than another, but we can consider a number of criteria to evaluate simplicity. For example, we can compare the number of sums and fractions and the number of variables present in the expression.

In this paper we propose a number of graphical criteria to infer which variables in our causal model are in fact not necessary for identification. These criteria allow us to prune the graph, which in practice means removing specific vertices and considering identification in a latent projection. A significantly simpler expression can be obtained by pruning alone, but we may also combine pruning with simplification procedures that operate symbolically on the interventional distribution as presented in \citep{tikka17b}. Applying these methods in conjunction often provides additional benefits.

We present an identifiability algorithm that is able to recognize and eliminate unnecessary variables from the graph based on our criteria resulting in a simpler expression. When a large number of graphs and identifiability queries are processed, evaluating simpler expressions has apparent computational benefits. First, it is more efficient to evaluate a simpler expression repeatedly especially when some variables have been completely removed which further reduces the complexity of the task. Second, in practical applications that involve real-world data, variables often contain missing data or are affected by bias. Obtaining expression that do not involve such variables can be of great benefit in estimation. Third, a simpler expression is easier to communicate.

An introductory example motivates the use of the improved algorithm. We are interested in the causal effect of $X$ on $Y$ in graph $G$ of Figure~\ref{fig:intro}(\subref{fig:intro_start}). Here, open circles denote unobserved variables. A more in-depth overview of graph theoretic concepts used in this paper is provided in Section~\ref{sect:definitions}.
\begin{figure}[t]
  \centering
  \begin{subfigure}[t]{0.55\linewidth}
\begin{tikzpicture}[xscale=0.98, yscale = 1.2]

\node [dot = {0}{0}{W_1}{above left}] at (0.5,2) {};
\node [dot = {0}{0}{W_2}{below left}] at (0.5,0) {};
\node [dot = {0}{0}{Z_2}{above}] at (2.75,2) {};
\node [dot = {0}{0}{X}{below}] at (2,0) {};
\node [dot = {0}{0}{Z_3}{below}] at (4,1.25) {};
\node [dot = {0}{0}{Z_4}{below} ] at (2.66,0.45) {};
\node [dot = {0}{0}{Z_1}{above right}] at (4.5,0) {};
\node [dot = {0}{0}{Y}{below right}] at (6,0) {};

\node [lat = {0}{0}{U_1}{left}] at (0,1) {};
\node [lat = {0}{0}{U_2}{below}] at (1,0.75) {};
\node [lat = {0}{0}{U_3}{below}] at (1.25,-0.33) {};
\node [lat = {0}{0}{U_4}{below}] at (4,-0.33) {};
\node [lat = {0}{0}{U_5}{above left}] at (2,1.1) {};
\node [lat = {0}{0}{U_6}{above right}] at (4.8,1.45) {};
\node [lat = {0}{0}{U_7}{above}] at (2.85,1.1) {};
\node [lat = {0}{0}{U_8}{above}] at (3.55,0.45) {};

\draw [->] (W_1) -- (W_2);
\draw [->] (W_1) -- (X);
\draw [->] (W_2) -- (X);
\draw [->] (X)   -- (Z_1);
\draw [->] (Z_1) -- (Y);
\draw [->] (Z_2) -- (X);
\draw [->] (Z_2) -- (Z_1);
\draw [->] (Z_4) -- (Z_1);
\draw [->] (Z_2) -- (Z_3);
\draw [->] (Z_3) -- (Y);

\draw[->,dashed] (U_1) to (W_1);
\draw[->,dashed] (U_1) to (W_2);

\draw[->,dashed] (U_2) to (W_1);
\draw[->,dashed] (U_2) to (X);

\draw[->,dashed] (U_3) to (W_2);
\draw[->,dashed] (U_3) to (X);

\draw[->,dashed] (U_4) to (Y);
\draw[->,dashed] (U_4) to (X);

\draw[->,dashed] (U_5) to (Z_2);
\draw[->,dashed] (U_5) to (X);

\draw[->,dashed] (U_6) to (Z_2);
\draw[->,dashed] (U_6) to (Y);

\draw[->,dashed] (U_7) to (X);
\draw[->,dashed] (U_7) to (Z_3);

\draw[->,dashed] (U_8) to (Z_4);
\draw[->,dashed] (U_8) to (Z_1);

\end{tikzpicture}
  \caption{Graph $G$.}
  \label{fig:intro_start}
  \end{subfigure}
  \begin{subfigure}[t]{0.38\linewidth}
\begin{tikzpicture}[xscale=0.98, yscale=1.2]

\node [dot = {0}{0}{Z_2}{above}] at (2.75,2) {};
\node [dot = {0}{0}{X}{below}] at (2,0) {};
\node [dot = {0}{0}{Z_1}{above right}] at (4.5,0) {};
\node [dot = {0}{0}{Y}{below right}] at (6,0) {};

\node [lat = {0}{0}{U_4}{below}] at (4,-0.33) {};
\node [lat = {0}{0}{U_5}{above left}] at (2,1.1) {};
\node [lat = {0}{0}{U_6}{above right}] at (4.8,1.45) {};

\draw [->] (X)   -- (Z_1);
\draw [->] (Z_1) -- (Y);
\draw [->] (Z_2) -- (X);
\draw [->] (Z_2) -- (Z_1);
\draw [->] (Z_2) -- (Y);

\draw[->,dashed] (U_4) to (Y);
\draw[->,dashed] (U_4) to (X);

\draw[->,dashed] (U_5) to (Z_2);
\draw[->,dashed] (U_5) to (X);

\draw[->,dashed] (U_6) to (Z_2);
\draw[->,dashed] (U_6) to (Y);

\end{tikzpicture}
  \caption{Graph $G$ after pruning.}
  \label{fig:intro_pruned}
  \end{subfigure}

  \caption{Graph $G$ before and after pruning for the introductory example.}
  \label{fig:intro}
\end{figure}
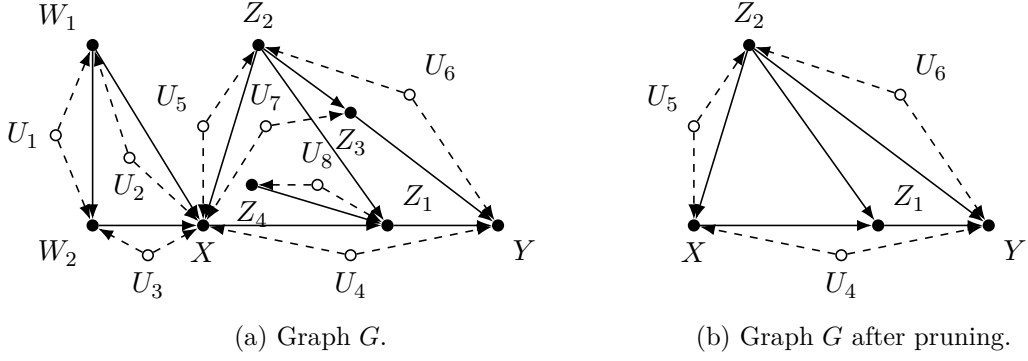
\noindent
The causal effect is identifiable and the output of the ID algorithm is

\begin{align*}
  \sum_{z_2,z_4,z_3,z_1}
    &\left(
       \sum_{w_1,w_2,x^\prime}
         P(y|w_1,z_2,z_4,w_2,z_3,x^\prime,z_1)
         P(x^\prime|w_1,z_2,z_4,w_2,z_3)
    \right.
    \times \\
    &\quad \left.
      \vphantom{\sum_{x^\prime}}
      P(z_3|w_1,z_2,z_4,w_2)
      P(w_2|w_1,z_2,z_4)
      P(z_2|w_1)
      P(w_1)
    \right)
    \left.
      \vphantom{\sum_{x^\prime}} \middle/
    \right.
    \\
    &\left(
      \sum_{w_1,w_2,x^\prime,y^\prime}
        P(y^{\prime}|w_1,z_2,z_4,w_2,z_3,x^\prime,z_1)
        P(x^{\prime}|w_1,z_2,z_4,w_2,z_3)
    \right.
    \times\\
    &\quad \left.
      \vphantom{\sum_{x^\prime}}
      P(z_3|w_1,z_2,z_4,w_2)
      P(w_2|w_1,z_2,z_4)
      P(z_2|w_1)
      P(w_1)
    \right)
    \times \\
    &\left(
      \sum_{w_1,w_2,z_3^\prime,x^\prime,y^\prime}
      P(y^\prime|w_1,z_2,z_4,w_2,z_3^\prime,x^\prime,z_1)
      P(x^\prime|w_1,z_2,z_4,w_2,z_3^\prime)
      \right. \, \times \\
    &\quad \left. 
      P(z_3^\prime|w_1,z_2,z_4,w_2)
      P(w_2|w_1,z_2,z_4)
      P(z_2|w_1)
      P(w_1)
    \vphantom{\sum_{x^\prime}}
    \right)
    \times \\
    &\qquad
    P(z_1|w_1,z_2,z_4,w_2,x)
    P(z_3|z_2)
    P(z_4).
\end{align*}
This expression is very cumbersome and complicated. However, it turns out that a simpler expression exists for the causal effect. By exploiting the structure of the graph and using standard probability calculus the following expression can be obtained
\[
  \sum_{z_2,z_1}\left(\sum_{x^{\prime}}P(y|z_2,z_1,x^{\prime})P(x^{\prime}|z_2)P(z_2)\right)P(z_1|z_2,x).
\]
This expression is simpler in every regard compared to the original output. It contains fewer terms and no fractions. Also, we have completely removed the variables $w_1, w_2$ and $z_4$ from the expression. It can be shown that identifying the causal effect in the original graph is equivalent to identifying it in the graph depicted in Figure~\ref{fig:intro}(\subref{fig:intro_pruned}). By running our improved algorithm we are able to prune the original graph and obtain this simpler expression directly. The algorithm works recursively and the pruning is carried out at each stage of the recursion. The recursive pruning provides significant benefits over pruning as a pre-processing step as demonstrated later.

The paper is structured as follows. In Section~\ref{sect:definitions} we review crucial definitions and concepts related to graph theory and causal models. In Section~\ref{sect:algorithm} we focus on semi-Markovian causal models and present the original formulation of the ID algorithm. Our main results are presented in Section~\ref{sect:improvements} and they are implemented into an improved identifiability algorithm in Section~\ref{sect:augmented}. Examples on the benefits of recursive pruning are provided in Section~\ref{sect:examples}. Section~\ref{sect:disc} concludes with a discussion.

\section{Definitions} \label{sect:definitions}

We assume the reader to be familiar with a number of graph theoretic concepts and refer them to works such as \citep{koller09}. We use capital letters to denote vertices and the respective variables, and small letters to denote their values. Bold letters are used to denote sets. A directed graph with a vertex set $\+ V$ and an edge set $\+ E$ is denoted by $\langle \+ V, \+ E \rangle$. For a graph $G = \langle \+ V, \+ E \rangle$ and a set of vertices $\+ W \subseteq \+ V$ the sets $\Pa(\+ W)_G, \Ch(\+ W)_G, \An(\+ W)_G$ and $\De(\+ W)_G$ denote a set that contains $\+ W$ in addition to its parents, children, ancestors and descendants in $G$, respectively. We also define the set $\Co(\+ W)_G$ to denote the set of vertices that are connected to $\+ W$ in $G$ via paths where the directionality of the edges is ignored, including $\+ W$. The root set of a graph $G$ is the set of vertices without any descendants $\{X \in V \mid \De(X)_G \setminus \{X\} = \emptyset \}$, where $\setminus$ denotes the set difference. A subgraph of a graph $G = \langle \+ V, \+ E \rangle$ induced by a set of vertices $\+ W \subset \+ V$ is denoted by $G[\+ W]$. This subgraph retains all edges $V \rightarrow W$ of $G$ such that $V,W \in \+ W$. The graph obtained from $G$ by removing all incoming edges of $\+ X$ and all outgoing edges of $\+ Z$ is written as $G_{\thickbar{\+ X},\thickubar{\+ Z}}$. To facilitate analysis of causal effects we must first define the probabilistic causal model \citep{pearl09}.

\begin{definition}[Probabilistic Causal Model]
A \emph{probabilistic causal model} is a quadruple
$$ M = \langle \+U,\+V, \+F, P(\+ u) \rangle,$$
where
\begin{enumerate}
\item{$\+U$ is a set of unobserved (exogenous) variables that are determined by factors outside the model.}
\item{$\+V$ is a set $\{V_1,V_2,\ldots,V_n\}$ of observed (endogenous) variables that are determined by variables in $\+U \cup \+ V$.}
\item{$\+F$ is a set of functions $\{f_{V_1},f_{V_2},\ldots,f_{V_n}\}$ such that each $f_{V_i}$ is a mapping from (the respective domains of) $\+U \cup (\+ V \setminus \{V_i\})$ to $V_i$, and such that the entire set $\+ F$ forms a mapping from $\+ U$ to $\+ V$.}
\item{$P(\+ u)$ is a joint probability distribution of the variables in the set $\+ U$.}
\end{enumerate}
\end{definition}

Each causal model induces a causal diagram which is a directed graph that provides a graphical means to convey our assumptions of the causal mechanisms involved. The induced graph is constructed by adding a vertex for each variable in $\+ U \cup \+ V$ and a directed edge from $V_i \in \+ U \cup \+ V$ into $V_j \in \+ V$ whenever $f_{V_j}$ is defined in terms of $V_i$.

Causal inference often focuses on a sub-class of models that satisfy additional assumptions: each $U \in \+ U$ appears in at most two functions of $\+ F$, the variables in $\+ U$ are mutually independent and the induced graph of the model is acyclic. Models that satisfy these additional assumptions are called \emph{semi-Markovian causal models}. A graph associated with a semi-Markovian model is called a \emph{semi-Markovian graph} (SMG). In SMGs every $U \in \+ U$ has at most two children. When semi-Markovian models are considered it is common not to depict background variables in the induced graph explicitly. Unobserved variables with exactly two children are not denoted as $V_i \leftarrow U \rightarrow V_j$ but as a bidirected edge $V_i \leftrightarrow V_j$ instead. Furthermore, unobserved variables with only one or no children are omitted entirely. We also adopt these abbreviations. For SMGs the sets $\Pa(\cdot)_G, \Ch(\cdot)_G, \An(\cdot)_G, \De(\cdot)_G$ and $\Co(\cdot)_G$ contain only observed vertices. Additionally, a subgraph $G[\+ W]$ of an SMG $G$ will also retain any bidirected edges between vertices in $\+ W$.

Any DAG can be associated with an SMG by constructing its \emph{latent projection} \citep{verma93}.

\begin{definition}[latent projection] Let $G = \langle \+ V \cup \+ L, \+ E \rangle$ be a DAG such that the vertices in $\+ V$ are observed and the vertices in $\+ L$ are latent. The \emph{latent projection} $L(G, \+ V)$ is a DAG $\langle \+ V, \+ E_L \rangle$, where for every pair of distinct vertices $Z,W \in \+ V$ it holds that:

\begin{enumerate}
 \item{$L(G, \+ V)$ contains an edge $Z \rightarrow W$ if there exists a directed path $Z \rightarrow \cdots \rightarrow W$ in $G$ on which every vertex except $Z$ and $W$ is in $\+ L$.}
 \item{$L(G, \+ V)$ contains an edge $Z \leftrightarrow W$ if there exists a path from $Z$ to $W$ in $G$ that does not contain the pattern $Z \rightarrow M \leftarrow W$ (a collider) and on which every vertex except $Z$ and $W$ is in $\+ L$ and the first edge has an arrowhead pointing into $W$ and the last edge has an arrowhead pointing into $Z$.}
\end{enumerate}

\end{definition}
From the construction it is easy to see that a latent projection is in fact an SMG. The induced graph of a probabilistic causal model can also be used to derive conditional independences among the variables in the model using a concept known as d-separation. We provide a definition for d-separation \citep{shpitser08} which takes into account the presence of bidirected edges and is thus suitable for SMGs.

\begin{definition}[d-separation] A path $P$ in an SMG $G$ is said to be d-separated by a set $\+ Z$ if and only if either
\begin{enumerate}
\item{$P$ contains one of the following three patterns of edges: $I \rightarrow M \rightarrow J$, $I \leftrightarrow M \rightarrow J$ or $I \leftarrow M \rightarrow J$, such that $M \in \+ Z$, or}
\item{$P$ contains one of the following three patterns of edges: $I \rightarrow M \leftarrow J$, $I \leftrightarrow M \leftarrow J$, $I \leftrightarrow M \leftrightarrow J$, such that $\De(M)_G \cap \+ Z = \emptyset$.}
\end{enumerate}
Disjoint sets $\+ X$ and $\+ Y$ are said to be d-separated by $\+ Z$ in $G$ if every path from $\+ X$ to $\+ Y$ is d-separated by $\+ Z$ in $G$.
\end{definition}

Whenever we can decompose the joint distribution of the observed variables $\+ V$ and the unobserved variables $\+ U$ as $P(\+ v, \+ u) = \prod_{W \in \+ V \cup \+ U} P(w|\Pa^*(w)_G)$, where $\Pa^*(\cdot)$ also contains the unobserved parents but not the argument itself, we say that $G$ is an I-map of $P(\+ v, \+ u)$ \citep{pearl09}. If sets $\+ X$ and $\+ Y$ are d-separated by $\+ Z$ in G, then $\+ X$ is independent of $\+ Y$ given $\+ Z$ in every $P$ for which $G$ is an I-map \citep{pearl88}. We use the notation of \citep{dawid79} to denote this d-separation and conditional independence statement as $(\+ X \indep \+ Y | \+ Z)_G$. It is clear that the graph induced by any semi-Markovian causal model is an I-map for the joint distribution $P(\+ v, \+ u)$ induced by the model.

Our interest lies in the effects of actions imposing changes to the model. An action that forces $\+ X$ to take a specific value $\+ x$ is called an \emph{intervention} and it is denoted by $\doo(\+ x)$ \citep{pearl09}. An intervention $\doo(\+ x)$ on a model $M$ creates a new sub-model, denoted by $M_{\+ x}$, where the functions in $\+ F$ that determine the value of $\+ X$ have been replaced with constant functions. The \emph{interventional distribution} of a set of variables $\+ Y$ in the model $M_{\+ x}$ is denoted by $P_{\+ x}(\+ y)$. This distribution is also known as the \emph{causal effect} of $\+ X$ on $\+ Y$.

Multiple causal models can share the same graph, and thus the same sub-model resulting from an intervention. The question is, are our assumptions encoded in the causal model sufficient to uniquely specify an interventional distribution of interest. This notion is captured by the following definition \citep{shpitser06}.

\begin{definition}[identifiability] Let $G = \langle \+ V, \+ E \rangle$ be an SMG and let $\+ X$ and $\+ Y$ be disjoint sets of variables such that $\+ X, \+ Y \subset \+ V$. The causal effect of $\+ X$ on $\+ Y$ is said to be \emph{identifiable} from $P$ in $G$ if $P_{\+ x}(\+ y)$ is uniquely computable from $P(\+ V)$ in any causal model that induces $G$.
\end{definition}

In order to show the identifiability of a given effect we have to express the interventional distribution in terms of observed probabilities only. The link between observed probabilities and interventional distributions is provided by three inference rules known as \emph{do-calculus} \citep{pearl95}:

\begin{enumerate}
\item{Insertion and deletion of observations: 
$$P_{\+x}(\+ y|\+ z, \+ w) = P_{\+x}(\+ y| \+ w), \text{ if } (\+ Y \indep \+Z|\+X, \+ W)_{G_{\thickbar{\+ X}}}.$$}
\item{Exchanging actions and observations:
$$P_{\+x,\+ z}(\+ y|\+ w) = P_{\+x}(\+ y|\+ z, \+ w), \text{ if } (\+ Y \indep \+Z|\+X, \+ W)_{G_{\thickbar{\+ X},\thickubar{\+ Z}}}.$$ }
\item{Insertion and deletion of actions:
$$P_{\+x,\+ z}(\+ y|\+ w) = P_{\+x}(\+ y|\+ w), \text{ if } (\+ Y \indep \+Z|\+X, \+ W)_{G_{\thickbar{\+ X},\widebar{Z(\+ W)}}},$$ 
where $ Z(\+W) = \+Z \setminus \An(\+ W)_{G_{\thickbar{\+ X}}}.$ }
\end{enumerate}

Completeness of do-calculus was established independently by \citet{huang06} and \citet{shpitser06}. In this paper we focus on the solution provided by \citet{shpitser06}. They constructed an identifiability algorithm called ID, which in essence applies the rules of do-calculus and breaks the problem into smaller sub-problems repeatedly.

\section{ID Algorithm} \label{sect:algorithm} 

In order to present the ID algorithm, we first need some additional definitions that are used to construct the graphical criterion for non-identifiability \citep{shpitser06}.

\begin{definition}[C-component] Let $G$ be an SMG and let $C \subseteq G$. If every pair of vertices in $C$ is connected by a bidirected path, that is a path consisting entirely of bidirected edges, then $C$ is a \emph{C-component} (confounded component). Furthermore, $C$ is a \emph{maximal C-component} if $C$ contains every vertex connected to $C$ via bidirected paths in $G$ and $C$ is an induced subgraph of $G$.
\end{definition}

No restrictions are imposed on the directed edges of a C-component. The same is not true for the maximal C-components (also known as districts) of an SMG $G$, which are assumed to be induced subgraphs of $G$. This requirement guarantees the uniqueness of the maximal C-components.

Maximal C-components are an important tool for identifying causal effects. The set of maximal C-components of a semi-Markovian graph $G$ is denoted by $C(G)$. A result in \citep{tian02phd} states that if $C = \langle \+ C, \+ E \rangle$ is a maximal C-component and $C \subset G$ then the causal effect $P_{\+ v \setminus \+ c}(\+ c)$ is identifiable from $P$ in $G$. A distribution $P$ of a semi-Markovian model also factorizes with respect to the maximal C-components of the induced graph $G$ such that $P(\+ v) = \prod_{\langle \+ C, \+ E \rangle \in C(G)} P_{\+ v \setminus \+ c}(\+ c)$ \citep{shpitser06}. It is precisely this factorization that the ID algorithm takes advantage of. A specific type of C-component is used to characterize problematic structures for identifiability.

\begin{definition}[C-forest] Let $G$ be an SMG and let $\+ Y$ be the root set of $G$. If $G$ is a C-component and all observed vertices have at most one child, then $G$ is a $\+ Y-$\emph{rooted C-forest}.
\end{definition}

The complete criterion for non-identifiability uses a structure formed by two C-forests:

\begin{definition}[hedge] Let $\+ X, \+ Y \subset \+ V$ be disjoint sets of variables and let $G$ be an SMG. Let $F = \langle \+ V_F, \+ E_F \rangle$ and $F^\prime = \langle \+ V_{F^\prime}, \+ E_{F^\prime} \rangle$ be $\+ R$-rooted C-forests in $G$ such that $\+ V_F \cap \+ X \neq \emptyset$, $\+ V_{F^\prime} \cap \+ X = \emptyset$, $F^\prime \subseteq F$, and $\+ R \subseteq \An(\+ Y)_{G_{\thickbar{\+ X}}}$. Then $F$ and $F^\prime$ form a hedge for $P_{\+ x}(\+ y)$ in $G$.
\end{definition}

Intuitively hedges are a difficult concept. Whenever a hedge is present, there exists two causal models with the same probability distribution over $\+ V$ but their interventional distributions do not agree. Observational data can not be used to estimate causal effects in this scenario. We are now ready to present the ID algorithm.

\begin{algorithm}[t]
  \begin{algorithmic}[1]
    \INPUT{Value assignments $\+ x$ and $\+ y$, joint distribution $P(\+ v)$ and an SMG $G = \langle \+ V, \+ E \rangle$. $G$ is an $I$-map of $P$.}
    \OUTPUT{Expression for $P_{\+ x}(\+ y)$ in terms of $P(\+ v)$ or \textbf{FAIL}$(F,F^\prime)$.}
    \Statex
    \Statex \textbf{function}{ \textbf{ID}$(\+ y, \+ x, P, G)$ }
      \State \textbf{if}{ $\+ x = \emptyset$, }
      \Statex \quad \textbf{return} $\sum_{v \in \+ v \setminus \+ y}P(\+ v)$.
      \State \textbf{if}{ $\+ V \neq \An(\+ Y)_G$, }
        \Statex \quad \textbf{return}{ \textbf{ID}$(\+ y, \+ x \cap \An(\+ y)_G, P(\An(\+ Y)_G), G[\An(\+ Y)_G])$.}
      \State \textbf{let}{ $\+ W = (\+ V \setminus \+ X) \setminus \An(\+ Y)_{G_{\thickbar {\+ X}}}$.} 
      \Statex \textbf{if}{ $\+ W \neq \emptyset $, }
      \Statex \quad \textbf{return} \textbf{ID}$(\+ y, \+ x \cup \+ w, P, G)$.
      \State \textbf{if}{ $C(G[\+ V \setminus \+ X]) = \{G[\+ S_1], \ldots,G[\+ S_k]\}$, }
        \Statex \quad \textbf{return}{ $\sum_{v \in \+ v \setminus (\+ y \cup \+ x)} \prod_{i=1}^k$ \textbf{ID}($\+ s_i, \+ v \setminus \+ s_i, P, G)$.}
      \Statex \textbf{if}{ $C(G[\+V \setminus \+ X]) = \{G[\+ S]\}$,}
        \State \quad \textbf{if}{ $C(G) = \{G\}$,}
        \Statex \quad \quad \textbf{throw FAIL}{$(G, G[\+ S])$.}
        \State \quad \textbf{if}{ $G[\+ S] \in C(G)$,}
        \Statex \quad \quad \textbf{return} {$\sum_{v \in \+ s \setminus \+ y} \prod_{V_i \in \+ S}{P(v_i \vert v_\pi^{(i-1)})}$.}
        \State \quad \textbf{if}{ $(\exists \+ S^\prime)\+ S \subset \+ S^\prime \textbf{ such that } G[\+ S^\prime] \in C(G)$,}
        \Statex \quad \quad \textbf{return} {\textbf{ID}$(\+ y, \+ x \cap \+ s^\prime, \prod_{V_i \in \+ S^\prime}{P(V_i \vert V_\pi^{(i-1)} \cap \+ S^\prime,v_\pi^{(i-1)} \setminus \+ s^\prime), G[\+ S^\prime}])$.}
  \end{algorithmic}
  \caption{The causal effect of intervention $do(\+ X = \+ x)$ on $\+ Y$ (ID).}
  \label{alg:identify}
\end{algorithm}

\citet{shpitser06} showed that whenever Algorithm~\ref{alg:identify} returns an expression for a causal effect, it is correct. Additionally whenever line 5 is triggered there exists a hedge for the causal effect currently being identified. This result establishes the completeness of the algorithm and also the completeness of do-calculus, since the soundness of each line of the algorithm can be shown with do-calculus and standard probability calculus alone.

\section{Pruning of Variables} \label{sect:improvements}
In this section we present a number of results that deal with variables that are not necessary for identification either by removing them from the graph or by considering them latent. When the causal effect $P_{\+ x}(\+ y)$ is considered in an SMG we can present an outline of the pruning process:
\begin{enumerate}
  \item Removal of non-ancestors of $\+ Y$.
  \item Removal of ancestors of $\+ X$ that are connected to $\+ Y$ only via $\+ X$ under certain conditions.
  \item Removal of vertices connected to other vertices only through a single vertex.
  \item Identification in a latent projection under certain conditions.
\end{enumerate}
Steps 2--4 are new and they are based on the results of this section. Step 1 is derived from a useful result by \citet{shpitser06} which states that for a causal effect $P_{\+ x}(\+ y)$ we can always ignore non-ancestors of $\+ Y$. 

\begin{lemma} \label{lem:ancestors} Let $\+ X^\prime = \+ X \cap \An(\+ Y)_G$. Then $P_{\+ x}(\+y)$ obtained from $P$ in $G$ is equal to $P_{\+ x^\prime}^\prime(\+ y)$ obtained from $P^\prime = P(\An(\+ Y)_G)$ in $G[\An(\+ Y)_G]$.
\end{lemma}

Lemma~\ref{lem:ancestors} is implemented on line 2 of Algorithm~\ref{alg:identify}.
Not all ancestors of $\+ Y$ are always necessary for identification. The next result states that we may sometimes remove ancestors of $\+ X$ that are connected to $\+ Y$ only through $\+ X$.

\begin{theorem} \label{thm:Xancestors} Let $G$ be an SMG and let $\+ Z \subset \+ V$ be the set of all vertices such that $\+ X$ intercepts all paths from $\+ Z$ to $\+ Y$. Then the causal effect $P_{\+ x}(\+ y)$ obtained from $P$ in $G$ is equal to $P^\prime_{\+ x}(\+ y)$ obtained from  $P^\prime = P(\+ V \setminus \+ Z)$ in $G[\+ V \setminus \+ Z]$ if $\+ Z$ contains no members of $\+ X$ and if $G[\+ V \setminus \+ Z] = L(G, \+ V \setminus \+ Z)$.
\end{theorem}
\begin{proof} 
Let $G^\prime = G[\+ V \setminus \+ Z]$ and assume that $G^\prime = L(G, \+ V \setminus \+ Z)$. Let $\+ U_{\+ Z}$, $\+ U_{\+ V \setminus \+ Z}$ and $\+ U_{\+ X}$ be sets of unobserved variables such that for all $U \in \+ U_{\+ Z}$ it holds that $\Ch(U)_{G_{\thickbar{\+ X}}} \subseteq \+ Z$, for all $U \in \+ U_{\+ V \setminus \+ Z}$ it holds that $\Ch(U)_{G_{\thickbar{\+ X}}} \subseteq \+V \setminus \+ Z$ and for all $U \in \+ U_{\+ X}$ it holds that $\Ch(U)_G \in \+ X$. The sets $\+ U_{\+ Z}$, $\+ U_{\+ V \setminus \+ Z}$ and $\+ U_{\+ X}$ partition $\+ U$ because $\+ X$ intercepts all paths from $\+ Z$ to $\+ Y$. 
According to the third rule of do-calculus $P_{\+ x}(\+ y) = P_{\+ x, \+ z}(\+ y)$ because the condition $(\+ Y \indep \+ Z \vert \+ X)_{G_{\thickbar{\+ X}}}$ holds as removing the edges incoming to $\+ X$ separates $\+ X$ from its ancestors. Applying the truncated factorization formula \citep{pearl09} we have that
\[
  P_{\+ x, \+ z}(\+ y) = \sum_{\+ U} \sum_{\+ V \setminus (\+ Y \cup \+ X \cup \+ Z)} \prod_{\+ V \setminus (\+ X \cup \+ Z)} P(v_i|\Pa(v_i)_G \setminus \{v_i\}) \prod_{\+ U} P(u_i).
\]
Since variables in $\+ U_{\+ Z}$ can only be parents of variables in $\+ Z$ or $\+ X$ in $G$, we can sum them out from the previous expression and obtain
\[
  P_{\+ x, \+ z}(\+ y) = \sum_{\+ U_{\+ V \setminus \+ Z} \cup \+ U_{\+ X}} \sum_{\+ V \setminus (\+ Y \cup \+ X \cup \+ Z)} \prod_{\+ V \setminus (\+ X \cup \+ Z)} P(v_i|\Pa(v_i)_G \setminus \{v_i\}) \prod_{\+ U_{\+ V \setminus \+ Z} \cup \+ U_{\+ X}} P(u_i).
\]
Similarly, variables in $\+ U_{\+ X}$ can only be parents of variables in $\+ X$ in $G$, so we can also sum them out of the expression to obtain
\[
  P_{\+ x, \+ z}(\+ y) = \sum_{\+ U_{\+ V \setminus \+ Z}} \sum_{\+ V \setminus (\+ Y \cup \+ X \cup \+ Z)} \prod_{\+ V \setminus (\+ X \cup \+ Z)} P(v_i|\Pa(v_i)_G \setminus \{v_i\}) \prod_{\+ U_{\+ V \setminus \+ Z}} P(u_i).
\]
We let $\+ V^\prime = \+ V \setminus \+ Z$. \citet{verma93} showed that a graph and its latent projection have the same set of conditional independence relations among the observed variables. Because we have assumed that $G^\prime = L(G, \+ V \setminus \+ Z)$ every conditional independence between variables in $\+ V^\prime$ and $\+ U_{\+ V^\prime}$ applies in both $G$ and $G^\prime$. We have that for all $V_i \in \+ V^\prime \setminus \+ X$ it holds that $P(v_i|\Pa(v_i)_G \setminus \{v_i\}) = P^\prime(v_i|\Pa(v_i)_{G^\prime} \setminus \{v_i\})$ and for all $U_i \in \+U_{\+ V^\prime}$ it holds that $P(u_i) = P^\prime(u_i)$. Finally we obtain
\[
  P_{\+ x, \+ z}(\+ y) = \sum_{\+ U_{\+ V^\prime}} \sum_{\+ V^\prime \setminus (\+ Y \cup \+ X)} \prod_{\+ V^\prime \setminus \+ X} P^\prime(v_i|\Pa(v_i)_{G^\prime} \setminus \{v_i\}) \prod_{\+ U_{\+ V^\prime}} P^\prime(u_i) = P^\prime_{\+ x}(\+ y).
\]
\end{proof}
Theorem~\ref{thm:Xancestors} can also be applied in a more general setting where a subset of $\+ X$ intercepts all paths from a set $\+ Z$ to $\+ Y$
\begin{corollary} \label{cor:extend}
Let $G$ be an SMG  and let $\+ Z \subset \+ V$ be the set of all vertices such that a set $\+ W \subseteq \+ X$ intercepts all paths from $\+ Z$ to $\+ Y$ and no member of $\+ X \setminus \+ W$ is a descendant of $\+ W$. Then the causal effect $P_{\+ x}(\+y)$ obtained from $P$ in $G$ is equal to $P^\prime_{\+ x \setminus \+ z}(\+ y)$ obtained from  $P^\prime = P(\+ V \setminus \+ Z)$ in $G[\+ V \setminus \+ Z]$ if $\+ Z$ contains no members of $\+ W$ and if $G[\+ V \setminus \+ Z] = L(G, \+ V \setminus \+ Z)$.
\end{corollary}
\begin{proof}
Since $\+ W$ intercepts all paths from $\+ Z$ to $\+ Y$ and no member of $\+ X \setminus \+ W$ is a descendant of $\+ W$, it follows that no member of $\+ W$ is in $\+ Z$. According to the third rule of do-calculus we have that $(\+ Y \indep \+ Z|\+ X \setminus \+ Z)_{G_{\overline{\+ X \setminus \+ Z}}}$ and $P_{\+ x}(\+ y) = P_{\+ x \setminus \+ z}(\+ y)$. The claim now follows by applying Theorem~\ref{thm:Xancestors} to $P_{\+ x \setminus \+ z}(\+ y)$.
\end{proof}

\begin{corollary} \label{cor:removal}
When the causal effect $P_{\+ x}(\+ y)$ is considered in graph $G$, a set of vertices $\+ Z = \An(\+ Y)_{G} \setminus \Co(\+ Y)_{G_{\thickbar{\+ X}}}$ can be removed from $G$ if $G[\+ V \setminus \+ Z] = L(G, \+ V \setminus \+ Z)$.
\end{corollary}
\begin{proof} 
The set $\An(\+ Y)_{G}$ contains $\+ Y$ in addition to the ancestors of $\+ Y$, and the set $\Co(\+ Y)_{G_{\thickbar{\+ X}}}$ contains $\+ Y$ and all vertices that are connected to $\+ Y$ via a path that does not contain edges incoming to $\+ X$. Therefore, $\+ Z$ contains such ancestors of $\+ Y$ that all paths from $\+ Z$ to $\+ Y$ contain $\+ X$. The removal of $\+ Z$ from $G$ is now licensed by Corollary~\ref{cor:extend}.
\end{proof}

Corollary~\ref{cor:removal} provides a constructive criterion for the set $\+ Z$ described in Corollary~\ref{cor:extend} when $G$ consists only of $\+ Y$ and its ancestors. If a vertex $Z_i$ is a member of $\An(\+ Y)_G \setminus \Co(\+ Y)_{G_{\thickbar{\+ X}}}$ then it must be connected to $\+ Y$ only through paths containing some $\+ W_{Z_i} \subseteq \+ X$. We can always choose the sets $\+ W_{Z_i}$ in such a way that the union $\+ W = \cup \+ W_{Z_i}$ over the members $Z_i$ of $\An(\+ Y)_G \setminus \Co(\+ Y)_{G_{\thickbar{\+ X}}}$ has no descendants in $\+ X \setminus \+ W$. The set $\+ W$ intercepts all paths from $\+ Z = \cup \{Z_i\}$ to $\+ Y$. Conversely, if $Z_i$ is a vertex such that a set $\+ W \subseteq \+ X$ intercepts all paths from $Z_i$ to $\+ Y$, then $Z_i$ cannot be connected to $\+ Y$ in $G_{\thickbar{\+ X}}$. If we assume that $G = G[\An(\+ Y)_G]$ it follows that $Z_i$ is a member of $\An(\+ Y)_G \setminus \Co(\+ Y)_{G_{\thickbar{\+ X}}}$.

We present a simple example to motivate the usefulness of Corollary~\ref{cor:removal}. We apply the ID algorithm to identify the causal effect of $X$ on $Y$ in graph $G$ of Figure~\ref{fig:cor1graph}(\subref{fig:cor1graph_start}).
\begin{figure}[h]
  \begin{subfigure}[t]{0.55\textwidth}
  \centering
\begin{tikzpicture}[scale=1.7]
\node [dot = {0}{0}{W_1}{above}] at (0.5,1.66) {};
\node [dot = {0}{0}{W_2}{below}] at (0,1) {};
\node [dot = {0}{0}{X}{below}] at (1,1) {};
\node [dot = {0}{0}{Z}{below}] at (2,1) {};
\node [dot = {0}{0}{Y}{below}] at (3,1) {};

\draw [->] (W_1) -- (W_2);
\draw [->] (W_1) -- (X);
\draw [->] (W_2) -- (X);
\draw [->] (X)   -- (Z);
\draw [->] (Z)   -- (Y);

\draw [<->,dashed] (W_1) to [bend right=45]  (W_2);
\draw [<->,dashed] (W_1) to [bend left=45]  (X);
\draw [<->,dashed] (X)   to [bend left=45]  (Y);
\end{tikzpicture}
  \caption{Graph $G$.}
  \label{fig:cor1graph_start}
  \end{subfigure}
  \begin{subfigure}[t]{0.4\textwidth}
  \centering
\begin{tikzpicture}[scale=1.7]
\node [dot = {0}{0}{X}{below}] at (1,1) {};
\node [dot = {0}{0}{Z}{below}] at (2,1) {};
\node [dot = {0}{0}{Y}{below}] at (3,1) {};

\draw [->] (X)   -- (Z);
\draw [->] (Z)   -- (Y);

\draw [<->,dashed] (X)   to [bend left=45]  (Y);
\end{tikzpicture}
  \caption{Subgraph $G[\{X,Y,Z\}]$.}
  \label{fig:cor1graph_sub}
  \end{subfigure}
  \caption{A graph for an example where Corollary~\ref{cor:removal} allows us to remove vertices $W_1$ and $W_2$ when the causal effect of $X$ on $Y$ is considered.}
  \label{fig:cor1graph}
\end{figure}
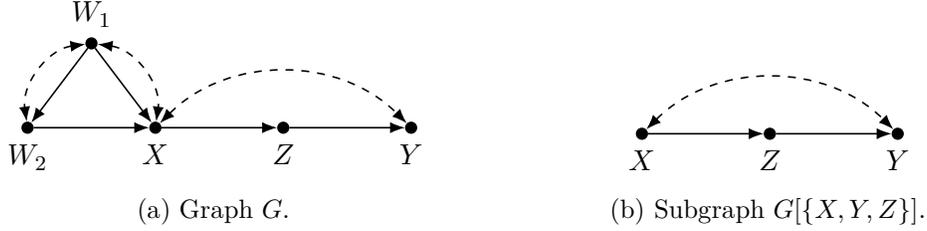
\noindent
Applying the ID algorithm results in the following expression for the causal effect
\[
  \sum_{z}P(z|w_1,w_2,x)\left(\sum_{w_1,w_2,x^\prime}P(y|w_1,w_2,x^\prime,z)P(x^\prime|w_1,w_2)P(w_2|w_1)P(w_1)\right).
\]
Applying Corollary~\ref{cor:removal} in this case would result in the removal of the vertices $W_1$ and $W_2$ from the graph, since they are ancestors of $Y$ in $G$ but not connected to $Y$ in $G_{\thickbar{X}}$ and the corresponding latent projection is the subgraph $G[\{X,Y,Z\}]$ of Figure~\ref{fig:cor1graph}(\subref{fig:cor1graph_sub}). Running the ID algorithm in this subgraph provides us the following expression
\[
  \sum_{z}P(z|x)\left(\sum_{x^\prime}P(y|x^\prime,z)P(x^\prime) \right).
\]
We may consider this expression simpler compared to the previous output by noting that $W_1$ and $W_2$ do not appear in the expression and it has fewer unique terms. The same expression can also be obtained manually by applying the front-door criterion \citep{pearl09}.

Often the question of identifiability can not be answered directly by neither the back-door nor the front-door criterion which leads us to more general methods, such as the ID algorithm. We are interested in the causal effect of $W_1$, $X_1$ and $X_2$ on $Y_1$ and $Y_2$ in the graph of Figure~\ref{fig:cor1graph2}(\subref{fig:cor1graph2_start}).

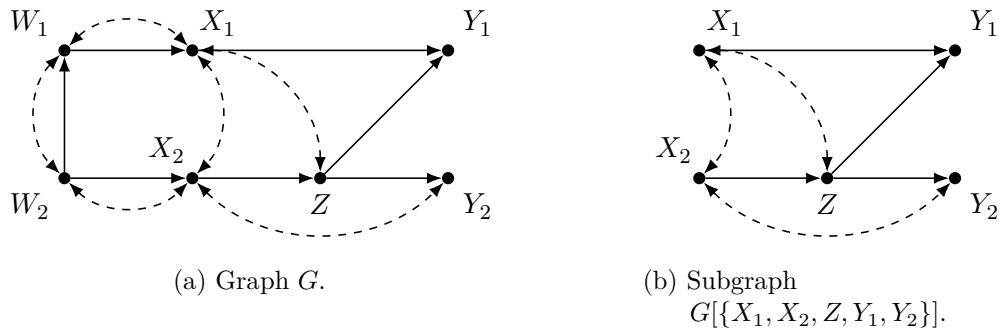
\begin{figure}[h]
  \begin{subfigure}[t]{0.55\textwidth}
  \centering
\begin{tikzpicture}[scale=1.7]
\node [dot = {0}{0}{W_1}{above left}] at (0,1) {};
\node [dot = {0}{0}{W_2}{below left}] at (0,0) {};
\node [dot = {-0.1}{0}{X_1}{above right}] at (1,1) {};
\node [dot = {0.1}{0}{X_2}{above left}] at (1,0) {};
\node [dot = {0}{0}{Z}{below}] at (2,0) {};
\node [dot = {0}{0}{Y_1}{above right}] at (3,1) {};
\node [dot = {0}{0}{Y_2}{below right}] at (3,0) {};

\draw [->] (W_1) -- (X_1);
\draw [->] (W_2) -- (W_1);
\draw [->] (W_2) -- (X_2);
\draw [->] (X_2) -- (Z);
\draw [->] (Z)   -- (Y_2);
\draw [->] (X_1) -- (Y_1);
\draw [->] (Z)   -- (Y_1);

\draw [<->,dashed] (W_1) to [bend right=45]  (W_2);
\draw [<->,dashed] (W_1) to [bend left=45]  (X_1);
\draw [<->,dashed] (W_2) to [bend right=45]  (X_2);
\draw [<->,dashed] (X_1) to [bend left=45]  (X_2);
\draw [<->,dashed] (X_1) to [bend left=45]  (Z);
\draw [<->,dashed] (X_2) to [bend right=45]  (Y_2);
\end{tikzpicture}
  \caption{Graph $G$.}
  \label{fig:cor1graph2_start}
  \end{subfigure}
  \hspace{0.8cm}
  \begin{subfigure}[t]{0.36\textwidth}
\begin{tikzpicture}[scale=1.7]
\node [dot = {-0.1}{0}{X_1}{above right}] at (1,1) {};
\node [dot = {0.1}{0}{X_2}{above left}] at (1,0) {};
\node [dot = {0}{0}{Z}{below}] at (2,0) {};
\node [dot = {0}{0}{Y_1}{above right}] at (3,1) {};
\node [dot = {0}{0}{Y_2}{below right}] at (3,0) {};

\draw [->] (X_2) -- (Z);
\draw [->] (Z)   -- (Y_2);
\draw [->] (X_1) -- (Y_1);
\draw [->] (Z)   -- (Y_1);

\draw [<->,dashed] (X_1) to [bend left=45]  (X_2);
\draw [<->,dashed] (X_1) to [bend left=45]  (Z);
\draw [<->,dashed] (X_2) to [bend right=45]  (Y_2);
\end{tikzpicture}
  \caption{Subgraph $G[\{X_1,X_2,Z,Y_1,Y_2\}]$.}
  \label{fig:cor1graph2_pruned}
  \end{subfigure}
  \caption{A graph for an example where the back-door and front-door criteria are unavailable, but Corollary~\ref{cor:removal} allows us to remove vertices $W_1$ and $W_2$ even when $W_1$ is part of the intervention.}
  \label{fig:cor1graph2}
\end{figure}
\noindent
Direct application of the ID algorithm provides us with the following expression
\[
  \sum_{z}P(y_1|w_2,w_1,x_2,x_1,z)P(z|w_2,x_2)\left(\sum_{w_2,x_2^{\prime}}P(y_2|w_2,x_2^{\prime},z)P(x_2^{\prime}|w_2)P(w_2)\right).
\]
Corollary~\ref{cor:removal} licenses the removal of $W_1$ and $W_2$ from the graph. By running ID again in the resulting subgraph $G[\{X_1,X_2,Z,Y_1,Y_2\}]$ as shown in Figure~\ref{fig:cor1graph2}(\subref{fig:cor1graph2_pruned}) we obtain a simpler expression for the causal effect
\[
  \sum_{z}P(y_1|x_1,x_2,z)P(z|x_2)\left(\sum_{x_2^{\prime}}P(y_2|x_2^{\prime},z)P(x_2^{\prime})\right).
\]

The next example illustrates the necessity of the assumption $G[\+ V \setminus \+ Z] = L(G, \+ V \setminus \+ Z)$ of Theorem~\ref{thm:Xancestors} and Corollary~\ref{cor:extend}. We are interested in the causal effect of $X_1$ and $X_2$ on $Y$ in graph $G$ of Figure~\ref{fig:necessary}(\subref{fig:necessary_start}).
\begin{figure}[h]
  \begin{subfigure}[t]{1.0\textwidth}
    \centering
\begin{tikzpicture}[scale=2.0]
\node [dot = {0}{0}{X_1}{above}] at (1,1.5) {};
\node [dot = {0}{0}{X_2}{below}] at (1,0.5) {};
\node [dot = {0}{0}{Z}{below}] at (1.75,1.25) {};
\node [dot = {0}{0}{W}{left}] at (0,1) {};
\node [dot = {0}{0}{Y}{right}] at (2.5,1) {};

\draw [->] (W) -- (X_1);
\draw [->] (W) -- (X_2);
\draw [->] (X_1) -- (Z);
\draw [->] (Z) -- (Y);
\draw [->] (X_2) -- (Y);

\draw [<->,dashed] (W) to [bend left=45]  (X_1);
\draw [<->,dashed] (X_1) to [bend left=45]  (Y);
\draw [<->,dashed] (X_2) to [bend left=45]  (Z);
\end{tikzpicture}
  \caption{Graph $G$.}
  \label{fig:necessary_start}
  \end{subfigure}
  \begin{subfigure}[t]{0.5\textwidth}
  \centering
\begin{tikzpicture}[scale=2.0]
\node [dot = {0}{0}{X_1}{above}] at (1,1.5) {};
\node [dot = {0}{0}{X_2}{below}] at (1,0.5) {};
\node [dot = {0}{0}{Z}{below}] at (1.75,1.25) {};
\node [dot = {0}{0}{Y}{right}] at (2.5,1) {};

\draw [->] (X_1) -- (Z);
\draw [->] (Z) -- (Y);
\draw [->] (X_2) -- (Y);

\draw [<->,dashed] (X_1) to [bend left=45]  (Y);
\draw [<->,dashed] (X_2) to [bend left=45]  (Z);
\end{tikzpicture}
  \caption{Subgraph $G[\{X_1,X_2,Z,Y\}]$.}
  \label{fig:necessary_sub}
  \end{subfigure}
  \begin{subfigure}[t]{0.5\textwidth}
  \centering
\begin{tikzpicture}[scale=2.0]
\node [dot = {0}{0}{X_1}{above}] at (1,1.5) {};
\node [dot = {0}{0}{X_2}{below}] at (1,0.5) {};
\node [dot = {0}{0}{Z}{below}] at (1.75,1.25) {};
\node [dot = {0}{0}{Y}{right}] at (2.5,1) {};

\draw [->] (X_1) -- (Z);
\draw [->] (Z) -- (Y);
\draw [->] (X_2) -- (Y);

\draw [<->,dashed] (X_1) to [bend left=45]  (Y);
\draw [<->,dashed] (X_2) to [bend left=45]  (Z);
\draw [<->,dashed] (X_1) to [bend right=45]  (X_2);
\end{tikzpicture}
  \caption{Latent projection $L(G,\{X_1,X_2,Z,Y\})$.}
  \label{fig:necessary_latent}
  \end{subfigure}
  \caption{An example where the assumptions of Corollary~\ref{cor:extend} are not met because the subgraph in panel (b) and the latent projection in panel (c) differ from each other.}
  \label{fig:necessary}
\end{figure}
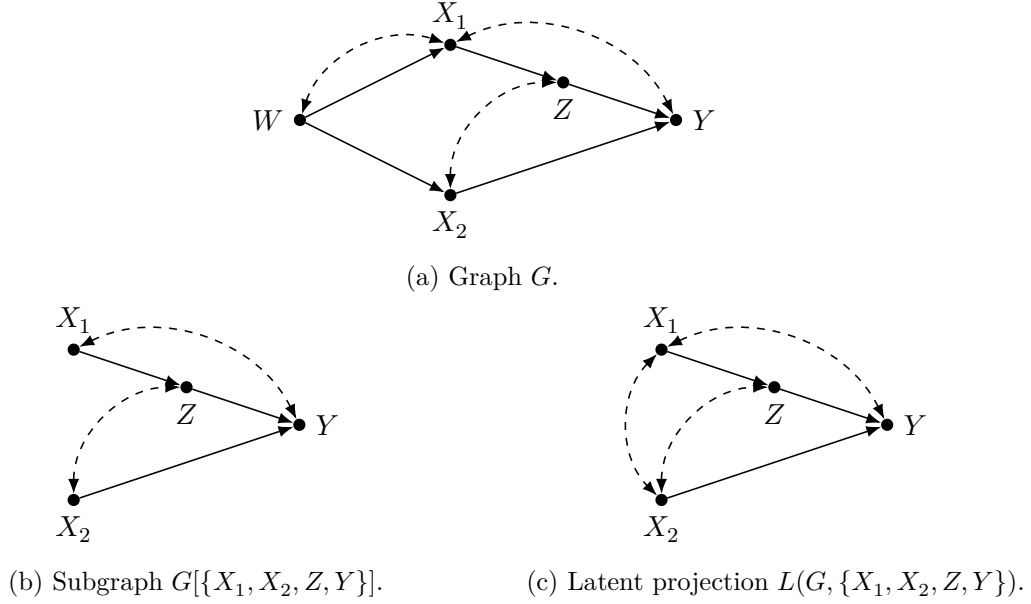
\noindent
In this graph $W$ is connected to $Y$ only through $X_1$ and $X_2$, but the corresponding latent projection does not match the subgraph with $W$ removed as seen in Figures~\ref{fig:necessary}(\subref{fig:necessary_sub}) and \ref{fig:necessary}(\subref{fig:necessary_latent}). In $G$ the causal effect is identifiable, but it is not identifiable in the latent projection $G^\prime = L(G, \{X_1,X_2,Z,Y\})$. In this latent projection a bidirected edge exists between $X_1$ and $X_2$ and a hedge is formed by the C-forests $G^\prime$ and $G[\{Y\}]$.

We may also remove sets of vertices that are connected to the rest of the graph only through a single vertex even when no intervention on the corresponding variable has taken place.

\begin{theorem} \label{thm:singleton}
Let $G$ be an SMG such that $G = G[An(\+ Y)_G]$ for a set of vertices $\+ Y$ and let $W$ be a vertex of $G$. If there exists a set $\+ Z$ such that $\+ Z \cap (\+ Y \cup \+ X) = \emptyset$ and $\+ Z$ is connected to $\+ V \setminus \+ Z$ only through $W$. Then the causal effect $P_{\+ x}(\+y)$ obtained from $P$ in $G$ is equal to $P^\prime_{\+ x}(\+ y)$ obtained from $P^\prime = P(\+ V \setminus \+ Z)$ in $G[\+ V \setminus \+ Z]$.
\end{theorem}
\begin{proof} Let $G^\prime = G[\+ V \setminus \+ Z]$ and let $\+ U_{\+ Z}$ and $\+ U_{\+ V \setminus \+ Z}$ be sets of unobserved variables such that for all $U \in \+ U_{\+ Z}$ it holds that $\Ch(U)_G \in \+ Z \cup \{W\}$ and for all $U \in \+ U_{\+ V \setminus \+ Z}$ it holds that $\Ch(U)_G \in (\+V \setminus \+ Z) \cup \{W\}$. Sets $\+ U_{\+ Z}$ and $\+ U_{\+ V \setminus \+ Z}$ partition $\+ U$ because $\+ Z$ is connected to $\+ V \setminus \+ Z$ only through $W$. Applying the truncated factorization formula yields
\[
  P_{\+ x}(\+ y) = \sum_{\+ U} \sum_{\+ V \setminus (\+ Y \cup \+ X)} P(w|\Pa(w)_G \setminus \{w\}) \prod_{\+ V \setminus (\+ X \cup \{W\})} P(v_i|\Pa(v_i)_G \setminus \{v_i\}) \prod_{\+ U} P(u_i).
\]
Since variables in $\+ U_{\+ Z}$ and $\+ Z$ can be connected to the other vertices of $G$ only through $W$ we can complete the marginalization over $\+ Z$ and $\+ U_{\+ Z}$
\[
  P_{\+ x}(\+ y) =\!\! \sum_{\+ U_{\+ V \setminus \+ Z}} \sum_{\+ V \setminus (\+ Y \cup \+ X \cup \+ Z)} \!P(w|\Pa(w)_G \setminus (\+ z \cup \{w\})) \prod_{\+ V \setminus (\+ X \cup \+ Z \cup \{W\})} \! P(v_i|\Pa(v_i)_G \setminus \{v_i\}) \! \prod_{\+ U_{\+ V \setminus \+ Z}} P(u_i).
\]
Because we have assumed that $\+ Z$ is disconnected from $\+ V \setminus \+ Z$ in $G_{\thickbar{W}}$ we have that $G[\+ V \setminus \+ Z] = L(G, \+ V \setminus \+ Z)$. Therefore, just as in the proof of Theorem~\ref{thm:Xancestors}, we have that $P(v_i|\Pa(v_i)_G \setminus \{v_i\}) = P^\prime(v_i|\Pa(v_i)_{G^\prime} \setminus \{v_i\})$ for all $V_i \in \+ V \setminus (\+ X \cup \+ Z \cup \{ W \})$ and $P(u_i) = P^\prime(u_i)$ for all $U_i \in \+U_{\+ V \setminus \+ Z}$. Additionally, we have $P(w|\Pa(w)_G \setminus (\+ z \cup \{w\})) = P(w|Pa(w)_{G^\prime} \setminus \{w\})$. Finally we obtain
\begin{align*}
  P_{\+ x}(\+ y) &= \sum_{\+ U_{\+ V \setminus \+ Z}} \sum_{\+ V \setminus (\+ Y \cup \+ X \cup \+ Z)} P(w|\Pa(w)_{G^\prime} \setminus \{w\}) \prod_{\+ V \setminus (\+ X \cup \+ Z \cup \{W\})} P(v_i|\Pa(v_i)_{G^\prime} \setminus \{v_i\}) \prod_{\+ U_{\+ V \setminus \+ Z}} P(u_i) \\
                 &= \sum_{\+ U_{\+ V \setminus \+ Z}} \sum_{\+ V \setminus (\+ Y \cup \+ X \cup \+ Z)} \prod_{\+ V \setminus (\+ X \cup \+ Z)} P(v_i|\Pa(v_i)_{G^\prime} \setminus \{v_i\}) \prod_{\+ U_{\+ V \setminus \+ Z}} P(u_i) \\
                 &= P_{\+ x}^\prime(\+ y).
\end{align*}
\end{proof}

\begin{corollary} \label{cor:removalsingle} Let $W$ be a vertex of an SMG $G$ and let $\+ R = \An(W)_{G_{\thickbar{\+ X}}} \setminus \De(\+ X)_G $. When the causal effect $P_{\+ x}(\+ y)$ is considered in graph $G$, the set of vertices $\+ T = \+ R \setminus \Co(\+ V \setminus \+ R)_{G_{\thickbar{W}}}$ can be removed from the graph if $G = G[\An(\+ Y)_G]$.
\end{corollary}
\begin{proof}
No descendant of $\+ X$ can be removed via theorem~\ref{thm:singleton} since they are connected to $\+ X$ and the set to be removed cannot itself contain $\+ X$. By removing descendants of $\+ X$ and $\+ X$ itself, and assuming that $G = G[\An(\+ Y)_G]$, we have that $\+ T \cap (\+ X \cup \+ Y) = \emptyset$. Thus it remains to remove those vertices from $\+ R$ that are connected to $\+ V \setminus \+ R$ through a path that does not contain $W$. Removal of the resulting set $\+ T$ from the graph is now licensed by theorem~\ref{thm:singleton}.
\end{proof}

In the following example we consider the causal effect of $X$ on $Y$ in graph $G$ of Figure~\ref{fig:cor2graph}(\subref{fig:cor2graph_start}) and show how Corollary~\ref{cor:removalsingle} can be applied.
\begin{figure}[h]
  \begin{subfigure}[t]{0.5\textwidth}
  \centering
\begin{tikzpicture}[scale=2.2]
\node [dot = {0}{0}{W_1}{below left}] at (1.5,0.6) {};
\node [dot = {0}{0}{W_2}{below right}] at (2.5,0.6) {};
\node [dot = {0}{0}{X}{above}] at (1,0) {};
\node [dot = {0}{0}{Z}{below}] at (2,0) {};
\node [dot = {0}{0}{Y}{above}] at (3,0) {};

\draw [->] (W_1) -- (W_2);
\draw [->] (W_1) -- (Z);
\draw [->] (W_2) -- (Z);
\draw [->] (X)   -- (Z);
\draw [->] (Z)   -- (Y);

\draw [<->,dashed] (W_1) to [bend left=45]  (W_2);
\draw [<->,dashed] (W_1) to [bend right=30]  (Z);
\draw [<->,dashed] (X)   to [bend right=45]  (Y);
\end{tikzpicture}
  \caption{Graph $G$.}
  \label{fig:cor2graph_start}
  \end{subfigure}
  \begin{subfigure}[t]{0.4\textwidth}
  \centering
\begin{tikzpicture}[scale=2.2]
\node [dot = {0}{0}{X}{above}] at (1,0) {};
\node [dot = {0}{0}{Z}{below}] at (2,0) {};
\node [dot = {0}{0}{Y}{above}] at (3,0) {};

\draw [->] (X) -- (Z);
\draw [->] (Z) -- (Y);

\draw [<->,dashed] (X) to [bend right=45]  (Y);
\end{tikzpicture}
  \caption{Subgraph $G[\{X,Y,Z\}]$.}
  \label{fig:cor2graph_sub}
  \end{subfigure}
  \caption{First example of Corollary~\ref{cor:removalsingle}.}
  \label{fig:cor2graph}
\end{figure}
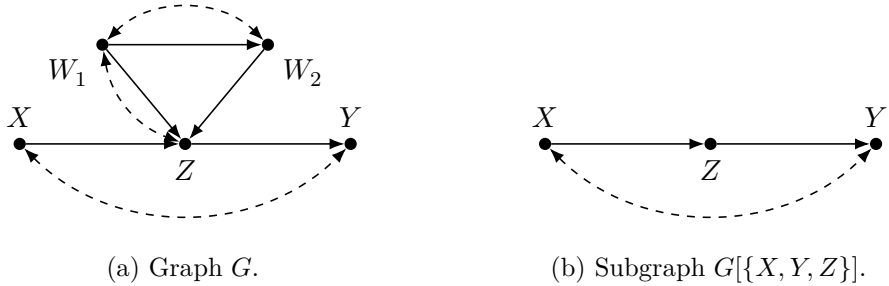
\noindent
The ID algorithm succeeds in identifying the causal effect and returns following expression for it
\[
  \sum_{w_1,w_2,z}P(z|x,w_1,w_2)P(w_2|x,w_1)P(w_1|x)\left(\sum_{x^\prime}P(y|x^\prime,w_1,w_2,z)P(x^\prime)\right).
\]
We apply Corollary~\ref{cor:removalsingle} which allows us to remove $W_1$ and $W_2$ from the graph, since they are connected to other vertices of the graph through a single vertex $Z$. Applying the ID algorithm in the resulting subgraph $G[\{X,Y,Z\}]$ of Figure~\ref{fig:cor2graph}(\subref{fig:cor2graph_sub}) provides us the following expression
\[
  \sum_{z}P(z|x)\left(\sum_{x^\prime}P(y|x^\prime,z)P(x^\prime) \right).
\]
The same expression can be obtained manually by applying the front-door criterion.

We provide another example on Corollary~\ref{cor:removalsingle} with a slightly more complicated graph. We are interested in the causal effect of $X$ on $Y$ in graph $G$ of Figure~\ref{fig:cor2graph2}(\subref{fig:cor2graph2_start})
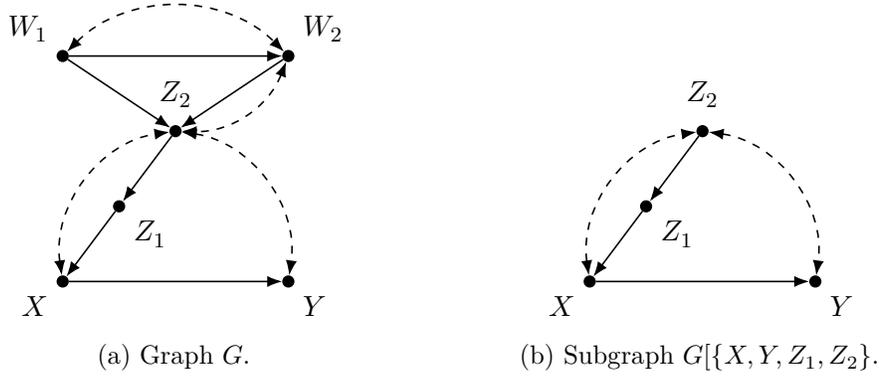
\begin{figure}[h]
  \begin{subfigure}[t]{0.5\textwidth}
  \centering
\begin{tikzpicture}
\node [dot = {0}{0}{W_1}{above left}] at (0,3) {};
\node [dot = {0}{0}{W_2}{above right}] at (3,3) {};
\node [dot = {0}{0}{X}{below left}] at (0,0) {};
\node [dot = {0}{0}{Z_1}{below right}] at (0.75,1) {};
\node [dot = {0}{0.1}{Z_2}{above}] at (1.5,2) {};
\node [dot = {0}{0}{Y}{below right}] at (3,0) {};

\draw [->] (W_1) -- (W_2);
\draw [->] (W_1) -- (Z_2);
\draw [->] (W_2) -- (Z_2);
\draw [->] (X)   -- (Y);
\draw [->] (Z_2) -- (Z_1);
\draw [->] (Z_1) -- (X);

\draw [<->,dashed] (W_1) to [bend left=45]  (W_2);
\draw [<->,dashed] (W_2) to [bend left=35]  (Z_2);
\draw [<->,dashed] (X)   to [bend left=45]  (Z_2);
\draw [<->,dashed] (Z_2) to [bend left=45]  (Y);
\end{tikzpicture}
  \caption{Graph $G$.}
  \label{fig:cor2graph2_start}
  \end{subfigure}
  \begin{subfigure}[t]{0.4\textwidth}
  \centering
\begin{tikzpicture}
\node [dot = {0}{0}{X}{below left}] at (0,0) {};
\node [dot = {0}{0}{Z_1}{below right}] at (0.75,1) {};
\node [dot = {0}{0.1}{Z_2}{above}] at (1.5,2) {};
\node [dot = {0}{0}{Y}{below right}] at (3,0) {};

\draw [->] (X)   -- (Y);
\draw [->] (Z_2) -- (Z_1);
\draw [->] (Z_1) -- (X);

\draw [<->,dashed] (X)   to [bend left=45]  (Z_2);
\draw [<->,dashed] (Z_2) to [bend left=45]  (Y);
\end{tikzpicture}
  \caption{Subgraph $G[\{X,Y,Z_1,Z_2\}$.}
  \label{fig:cor2graph2_sub}
  \end{subfigure}
  \caption{Second example of Corollary~\ref{cor:removalsingle}.}
  \label{fig:cor2graph2}
\end{figure}
\noindent
We obtain a formula for the causal effect using the ID algorithm
\[
  \frac{\sum_{w_2,z_2}P(y|w_1,w_2,z_2,z_1,x)P(x|w_1,w_2,z_2,z_1)P(z_2|w_1,w_2)P(w_2|w_1)}{\left(\sum_{w_2,z_2,y}P(y|w_1,w_2,z_2,z_1,x)P(x|w_1,w_2,z_2,z_1)P(z_2|w_1,w_2)P(w_2|w_1)\right)}.
\]
Vertices $W_1$ and $W_2$ are connected to other vertices only through $Z_2$ which allows us to remove them from the graph. We obtain a simpler formula for the causal effect from the subgraph $G[\{X,Y,Z_1,Z_2\}]$ of Figure~\ref{fig:cor2graph2}(\subref{fig:cor2graph2_sub})
\[
  \frac{\sum_{z_2}P(y|z_2,z_1,x)P(x|z_2,z_1)P(z_2)}{\sum_{z_2,y}P(y|z_2,z_1,x)P(x|z_2,z_1)P(z_2)}.
\]

The previous results have allowed us to completely remove specific vertices from the graph. Next we will consider cases where a vertex is present in the graph, but it is not necessary to observe it. This means that instead of the original graph we may consider identifiability in the corresponding latent projection, as characterized by the following lemma.

\begin{lemma} \label{lem:latent} Let $G = \langle \+ V, \+ E \rangle$ be an SMG and let $\+ X, \+ Y$ and $\+ Z$ be disjoint sets of variables. Let $P(\+ V)$ be the joint distribution of $\+ V$. Then the causal effect of $\+ X$ on $\+ Y$ is identifiable from $P^\prime$ in the latent projection $L(G, \+ V \setminus \+ Z)$ where $P^\prime = P(\+ V \setminus \+ Z)$, if and only if it is identifiable from $P^\prime$ in $G$.
\end{lemma}
\begin{proof} \citet{tian02phd} showed that the latent projection has the same topological relations over the observables and that it has the same set of maximal C-components. Thus if $P_{\+ x}(\+ y)$ is identifiable from $P^\prime$ in $G$ it is also identifiable from $P^\prime$ in $L(G, \+ V \setminus \+ Z)$ with the same expression and vice versa.
\end{proof}

In situations where $\+ X$ is a singleton we can exploit the following sufficient condition for identifiability by \citet{tian02}.
\begin{theorem} \label{thm:sufficent} The causal effect $P_{x}(\+ y)$ is identifiable if there is no bidirected path between $X$ and any of its children in $G[\An(\+ Y)_G]$.
\end{theorem}
We can regard any variable as latent when $\+ X$ is a singleton if the respective latent projection does not induce such a bidirected path. 
\begin{corollary} \label{cor:latent}
Let $G$ be an SMG and let $\+ Y$ be a set of vertices. Let $X, W \in \+ V$ such that $X \neq W$, $W \not\in \+ Y$ and $X, W \in \An(\+ Y)_G$. The causal effect $P_{x}(\+ y)$ obtained from $P$ in $G$ is equal to $P^\prime_{x}(\+ y)$ obtained from $P^\prime = P(\An(\+ Y)_G \setminus \{ w \})$ and $G^\prime = L(G[\An(\+ Y)_G], \An(\+ Y)_G \setminus \{ W \})$ if there is no bidirected path from $X$ to any of its children in $G^\prime$.
\end{corollary}
\begin{proof} 
By Theorem~\ref{thm:sufficent} the causal effect $P^\prime_{x}(\+ y)$ is identifiable from $P^\prime$ in $G^\prime$. By Lemma~\ref{lem:latent} $P^\prime_{x}(\+ y)$ is now identifiable from $P^\prime$ in $G$. $P_{x}(y)$ obtained from $P^\prime$ in $G$ is equal to $P^*_{x}(y)$ obtained from $P^* = P(\An(\+ Y)_G)$ in $G[\An(\+ Y)]$ since identifiability from $P^\prime$ implies identifiability from $P$. Finally, $P^*_{x}(\+ y)$ obtained from $P^* = P(\An(\+ Y)_G)$ in $G[\An(\+ Y)_G]$ is equal to $P_{x}(y)$ obtained from $P$ in $G$.
\end{proof}

We continue with an example on how Corollary~\ref{cor:latent} can be applied in practice. We consider the causal effect of $X$ on $Y$ in the graph $G$ of Figure~\ref{fig:cor3graph}(\subref{fig:cor3graph_start}).
\begin{figure}[h]
  \begin{subfigure}[t]{0.5\textwidth}
  \centering
\begin{tikzpicture}[scale=0.65]
\node [dot = {0}{0}{X}{above left}] at (0,4) {};
\node [dot = {0}{0}{Y}{below}] at (2,0) {};
\node [dot = {-0.15}{0.05}{Z_1}{above right}] at (1,2) {};
\node [dot = {0}{0}{Z_2}{above right}] at (7,4) {};
\node [dot = {0}{0}{Z_3}{below right}] at (4.5,2) {};

\draw [->] (X)   -- (Z_1);
\draw [->] (Z_1) -- (Y);
\draw [->] (Z_2) -- (X);
\draw [->] (Z_2) -- (Z_3);
\draw [->] (Z_2) -- (Z_1);
\draw [->] (Z_3) -- (Y);

\draw [<->,dashed] (X) to [bend left=45]  (Z_2);
\draw [<->,dashed] (X) to [bend right=45]  (Y);
\draw [<->,dashed] (X) to [bend left=45]  (Z_3);
\draw [<->,dashed] (Z_2) to [bend left=45]  (Y);
\end{tikzpicture}
  \caption{Graph $G$.}
  \label{fig:cor3graph_start}
  \end{subfigure}
  \hspace{0.5cm}
  \begin{subfigure}[t]{0.48\textwidth}
  \centering
\begin{tikzpicture}[scale=0.65]
\node [dot = {0}{0}{X}{above left}] at (0,4) {};
\node [dot = {0}{0}{Y}{below}] at (2,0) {};
\node [dot = {-0.15}{0.05}{Z_1}{above right}] at (1,2) {};
\node [dot = {0}{0}{Z_2}{above right}] at (7,4) {};

\draw [->] (X)   -- (Z_1);
\draw [->] (Z_1) -- (Y);
\draw [->] (Z_2) -- (X);
\draw [->] (Z_2) -- (Y);
\draw [->] (Z_2) -- (Z_1);

\draw [<->,dashed] (X) to [bend left=45]  (Z_2);
\draw [<->,dashed] (X) to [bend right=45]  (Y);
\draw [<->,dashed] (Z_2) to [bend left=45]  (Y);
\end{tikzpicture}
  \caption{Latent projection $L(G, \{X,Y,Z_1,Z_2\})$.}
  \label{fig:cor3graph_latent}
  \end{subfigure}
  \caption{A graph for an example where Corollary~\ref{cor:latent} allows us to make variable $Z_3$ latent.}
  \label{fig:cor3graph}
\end{figure}
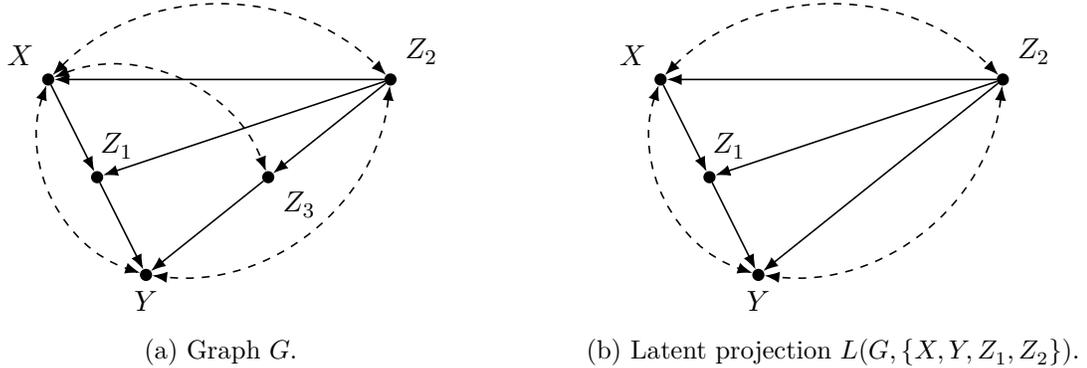
\noindent
The causal effect is identifiable and the output of the ID algorithm is 
\begin{align*}
&\sum_{z_2,z_3,z_1}P(z_1|z_2,x)\frac{\left(\sum_{x^\prime}P(y|z_2,x^\prime,z_3,z_1)P(z_3|z_2,x^\prime)P(x^\prime|z_2)P(z_2)\right)}{\left(\sum_{x^\prime,y^\prime}P(y^\prime|z_2,x^\prime,z_3,z_1)P(z_3|z_2,x^\prime)P(x^\prime|z_2)P(z_2)\right)} \times \\
  &\qquad\left(\sum_{x^\prime,z_3,y^\prime}P(y^\prime|z_2,x^\prime,z_3,z_1)P(z_3|z_2,x^\prime)P(x^\prime|z_2)P(z_2)\right)P(z_3|z_2).
\end{align*}
\noindent
We may apply Corollary~\ref{cor:latent} by noting that $G = G[\An(Y)_G]$ and that there is no bidirected path between $X$ and its only child $Z_1$ in the latent projection $L(G, \+ V \setminus \{Z_3\})$ as depicted in Figure~\ref{fig:cor3graph}(\subref{fig:cor3graph_latent}).

Running the ID algorithm in $L(G, \+ V \setminus \{Z_3\})$ results in the following expression
\[
  \sum_{z_2,z_1}\left(\sum_{x^\prime}P(y|z_2,x^\prime,z_1)P(x^\prime|z_2)P(z_2)\right)P(z_1|z_2,x).
\]

\section{Pruning Identifiability Algorithm} \label{sect:augmented}

Corollaries~\ref{cor:removal}, \ref{cor:removalsingle} and \ref{cor:latent} can be implemented as additional steps for the ID algorithm. For Algorithm~\ref{alg:aug_identify}, line 3 implements Corollary~\ref{cor:removal}, line 4 implements Corollary~\ref{cor:removalsingle} and line 5 implements Corollary~\ref{cor:latent}. Other lines are identical to the ID algorithm. This algorithm is provided by the R package \emph{causaleffect} which implements various causal inference algorithms such as the original ID algorithm \citep{tikka17a}. 

The ordering of the variables in the loop on line 4 has no effect on the resulting expression, but the ordering does matter on line 5. Choosing a different ordering may lead to a different expression. For example, when identifying the causal effect of $X$ on $Y$ in graph $G$ of Figure~\ref{fig:ordering} one may obtain either the back-door formula or the front-door formula by proceeding in either the topological ordering or the reverse-topological ordering of the vertices, respectively.

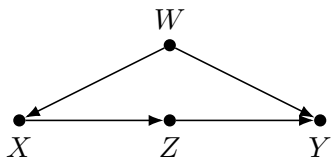
\begin{figure}[h]
  \centering
\begin{tikzpicture}[scale=2.0]
\node [dot = {0}{0}{X}{below}] at (0,0) {};
\node [dot = {0}{0}{Z}{below}] at (1,0) {};
\node [dot = {0}{0}{Y}{below}] at (2,0) {};
\node [dot = {0}{0}{W}{above}] at (1,0.5) {};

\draw [->] (X) -- (Z);
\draw [->] (Z) -- (Y);
\draw [->] (W) -- (X);
\draw [->] (W) -- (Y);
\end{tikzpicture}
  \caption{Graph $G$ where different latent projections lead to different expressions for $P_x(y)$.}
  \label{fig:ordering}
\end{figure}

\begin{algorithm}[H]
  \begin{algorithmic}[1]
  \INPUT{Value assignments $\+ x$ and $\+ y$, joint distribution $P(\+ v)$ and an SMG $G = \langle \+ V, \+ E \rangle$. $G$ is an $I$-map of $P$.}
    \OUTPUT{Expression for $P_{\+ x}(\+ y)$ in terms of $P(\+ v)$ or \textbf{FAIL}$(F,F^\prime)$.}
    \Statex
    \Statex \textbf{function}{ \textbf{PID}$(\+ y, \+ x, P, G)$ }
      \State \textbf{if}{ $\+ x = \emptyset$, }
      \Statex \quad \textbf{return} $\sum_{v \in \+ v \setminus \+ y}P(\+ v)$.
      \State \textbf{if}{ $\+ V \neq \An(\+ Y)_G$, }
        \Statex \quad \textbf{return} \textbf{PID}$(\+ y, \+ x \cap \An(\+ y)_G, P(\An(\+ Y)_G), G[\An(\+ Y)_G])$
      \State \textbf{let}{ $\+ Z = \An(\+ Y)_G \setminus \Co(\+ Y)_{G_{\thickbar{\+ X}}}$.}
      \Statex \textbf{if}{ $\+ Z \neq \emptyset $ } \textbf{and}{ $G[\+ V \setminus \+ Z] = L(G, \+ V \setminus \+ Z)$, }
        \Statex \quad \textbf{return} \textbf{PID}$(\+ y, \+ x \setminus \+ z, P(\+ V \setminus \+ Z), G[\+ V \setminus \+ Z])$
      \For{$W \in \+ V \setminus \+ X$}
        \Statex \quad \textbf{let}{ $\+ R = \An(W)_{G_{\thickbar{\+ X}}} \setminus \De(\+ X)_G$.}
        \Statex \quad \textbf{let}{ $\+ T = \+ T \cup (\+ R \setminus \Co(\+ V \setminus \+ R)_{G_{\thickbar{W}}})$.}
      \EndFor
        \Statex \textbf{if}{ $\+ T \neq \emptyset$, }
        \Statex \quad \textbf{return} \textbf{PID}$(\+ y, \+ x, P(\+ V \setminus \+ T), G[\+ V \setminus \+ T])$
      \State \textbf{if}{ $\+ X = \{X\}$, }
        \Statex \quad \textbf{let}{ $G[\+ S_X] \in C(G), X \in \+ S_X$.}
        \Statex \quad \textbf{if}{ $\Ch(\+ X)_{G[\+ S_X]} \setminus \+ X = \emptyset$,}
        \Statex \quad \quad \textbf{for}{ $W \in \+ V \setminus (\+ Y \cup \+ X)$ \textbf{do}}
        \Statex \quad \quad \quad \textbf{let}{ $G^\prime = L(G, \+ V \setminus \{W\}).$}
        \Statex \quad \quad \quad \textbf{let}{ $G^\prime[\+ S^\prime_X] \in C(G^\prime), X \in \+ S^\prime_X$.}
        \Statex \quad \quad \quad \textbf{if}{ $\Ch(\+ X)_{G^\prime[\+ S^\prime_X]} \setminus \+ X = \emptyset$, }
        \Statex \quad \quad \quad \quad $P \gets P(\+ V \setminus \{W\})$.
        \Statex \quad \quad \quad \quad $G \gets G^\prime$.
        \Statex \quad \quad \quad \quad $\+ V \gets \+ V \setminus \{W\}$.
      \State \textbf{let}{ $\+ W = (\+ V \setminus \+ X) \setminus \An(\+ Y)_{G_{\thickbar {\+ X}}}$.} 
      \Statex \textbf{if}{ $\+ W \neq \emptyset $, }
      \Statex \quad \textbf{return} \textbf{PID}$(\+ y, \+ x \cup \+ w, P, G)$.
      \State \textbf{if}{ $C(G[\+ V \setminus \+ X]) = \{G[\+ S_1], \ldots,G[\+ S_k]\}$, }
        \Statex \quad \textbf{return}{ $\sum_{v \in \+ v \setminus (\+ y \cup \+ x)} \prod_{i=1}^k$ \textbf{PID}($\+ s_i, \+ v \setminus \+ s_i, P, G)$.}
      \Statex \textbf{if}{ $C(G[\+V \setminus \+ X]) = \{G[\+ S]\}$,}
        \State \quad \textbf{if}{ $C(G) = \{G\}$,}
        \Statex \quad \quad \textbf{throw FAIL}{$(G, G[\+ S])$.}
        \State \quad \textbf{if}{ $G[\+ S] \in C(G)$,}
        \Statex \quad \quad \textbf{return} {$\sum_{v \in \+ s \setminus \+ y} \prod_{V_i \in \+ S}{P(v_i \vert v_\pi^{(i-1)})}$.}
        \State \quad \textbf{if}{ $(\exists \+ S^\prime)\+ S \subset \+ S^\prime \textbf{ such that } G[\+ S^\prime] \in C(G)$,}
        \Statex \quad \quad \textbf{return} {\textbf{PID}$(\+ y, \+ x \cap \+ s^\prime, \prod_{V_i \in \+ S^\prime}{P(V_i \vert V_\pi^{(i-1)} \cap \+ S^\prime,v_\pi^{(i-1)} \setminus \+ s^\prime), G[\+ S^\prime}])$.}
  \end{algorithmic}
  \caption{A pruning identifiability algorithm (PID) for causal effects.}
  \label{alg:aug_identify}
\end{algorithm}

On line 5 we first check whether the set $\+ X$ is a singleton. If so, we determine whether any children of $X$ belong the same C-component as $X$. If no such children exist, we iterate over the possible latent projections in an attempt to find one that does not induce a bidirected path between $X$ and any of its children in the projection. After the new pruning steps have been carried out, we attempt identification using the original formulation of the ID algorithm.


We return to the example presented in the introduction and show how Algorithm~\ref{alg:aug_identify} operates to derive the expression for $P_{x}(y)$ in the graph of Figure~\ref{fig:intro}(\subref{fig:intro_start}). We choose the topological ordering to be $Y > Z_1 > X > Z_3 > W_2 > Z_4 > Z_2 > W_1$. We begin on line 3, since $\+ Z = \An(Y)_G \setminus \Co(Y)_{G_{\thickbar{X}}} = \+ V \setminus \{W_1,W_2\}$ and continue by calling $\Call{PID}{y, x, P(\+ V \setminus \{W_1,W_2\}), G[\+ V \setminus \{W_1,W_2\}]}$. In the presentation below, $\+ V$ and $G$ refer to the set of vertices and graph in the current call of PID, respectively.

Next we enter the loop on line 4. When $W = Z_1$ we obtain $\+ R = \An(Z_1)_{G_{\thickbar{X}}} \setminus \De(X)_G = \{Z_2, Z_4\}$ and $\+ R \setminus \Co(\+ V \setminus \+ R)_{G_{\thickbar{Z_1}}} = \{Z_2, Z_4\} \setminus \{Z_2\} = \{Z_4\}$ since $Z_4$ is an ancestor of $Z_1$ and it is disconnected from other vertices in $G_{\thickbar{Z_1}}$. Other choices of $W$ result in an empty set. When the loop is completed we have $\+ T = \{Z_4\}$ and we continue by calling $\Call{PID}{y, x, P(\+ V \setminus \{Z_4\}), G[\+ V \setminus \{Z_4\}]}$. 

Since $\+ X$ is a singleton we end up on line 5 and find no children of $X$ in the same C-component as $X$. We assume a reverse-topological ordering for the loop and begin with the latent projection $L(G, \+ V \setminus \{Z_2\})$. This projection creates a bidirected edge between $X$ and $Z_1$ bringing them into the same C-component in the projection. Thus we continue with $L(G, \+ V \setminus \{Z_1\})$. This projection is also unsuitable, since $Y$ is a child of $X$ in the projection and there is a bidirected arc connecting them. We continue with $L(G, \+ V \setminus \{Z_3\})$ and find that $X$ is not connected to its children via bidirected paths in the projection. Thus we set $P \gets P(\+ V \setminus \{Z_3\})$, $G \gets L(G, \+ V \setminus \{Z_3\})$ and $\+ V \gets \+ V \setminus \{Z_3\}$. This projection is identical to the graph depicted in Figure~\ref{fig:cor3graph}(\subref{fig:cor3graph_latent}). After these steps, only lines of the original ID algorithm are called, which results in the expression
\[
  \sum_{z_2,z_1}\left(\sum_{x^{\prime}}P(y|z_2,z_1,x^{\prime})P(x^{\prime}|z_2)P(z_2)\right)P(z_1|z_2,x).
\]

\section{Examples on Recursive Pruning} \label{sect:examples}
Corollaries~\ref{cor:extend}, \ref{cor:removalsingle} and \ref{cor:latent} often provide direct benefits when applied before the ID algorithm. The following examples show why they are also useful as recursive steps as implemented in the PID algorithm.

We are tasked with identifying the causal effect of $X$ on $Y$ in graph $G$ depicted in Figure~\ref{fig:recursive1}(\subref{fig:recursive1_start})

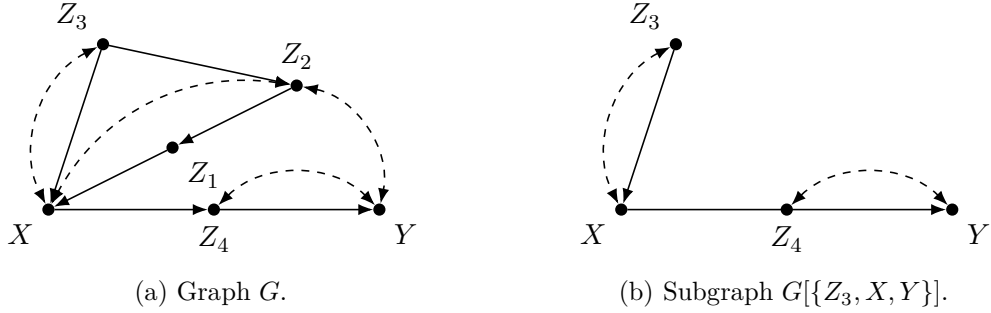
\begin{figure}[h]
  \begin{subfigure}[t]{0.5\textwidth}
  \centering
\begin{tikzpicture}[scale=2.2]
\node [dot = {0}{0}{Z_1}{below right}] at (0.75,0.375) {};
\node [dot = {0}{0}{Z_2}{above}] at (1.5,0.75) {};
\node [dot = {0}{0}{Z_3}{above left}] at (0.33,1) {};
\node [dot = {0}{0}{Z_4}{below}] at (1,0) {};
\node [dot = {0}{0}{X}{below left}] at (0,0) {};
\node [dot = {0}{0}{Y}{below right}] at (2,0) {};

\draw [->] (Z_2) -- (Z_1);
\draw [->] (Z_1) -- (X);
\draw [->] (X)   -- (Z_4);
\draw [->] (Z_4) -- (Y);
\draw [->] (Z_3) -- (Z_2);
\draw [->] (Z_3) -- (X);

\draw [<->,dashed] (Z_2) to [bend right=35]  (X);
\draw [<->,dashed] (Z_2) to [bend left=45]   (Y);
\draw [<->,dashed] (Z_3) to [bend right=45]  (X);
\draw [<->,dashed] (Z_4) to [bend left=45]  (Y);
\end{tikzpicture}
  \caption{Graph $G$.}
  \label{fig:recursive1_start}
  \end{subfigure}
  \hspace{0.5cm}
  \begin{subfigure}[t]{0.4\textwidth}
  \centering
\begin{tikzpicture}[scale=2.2]
\node [dot = {0}{0}{Z_3}{above left}] at (0.33,1) {};
\node [dot = {0}{0}{Z_4}{below}] at (1,0) {};
\node [dot = {0}{0}{X}{below left}] at (0,0) {};
\node [dot = {0}{0}{Y}{below right}] at (2,0) {};

\draw [->] (X)   -- (Y);
\draw [->] (Z_3) -- (X);

\draw [<->,dashed] (Z_3) to [bend right=45]  (X);
\draw [<->,dashed] (Z_4) to [bend left=45]  (Y);
\end{tikzpicture}
  \caption{Subgraph $G[\{Z_3,X,Y\}]$.}
  \label{fig:recursive1_sub}
  \end{subfigure}
  \caption{A graph for the example of recursive application of Corollary~\ref{cor:removal} within the ID algorithm.}
  \label{fig:recursive1}
\end{figure}
\noindent
Initially there are no vertices that are connected to other vertices only through $X$. As a recursive step of the ID algorithm, we are tasked with identifying
$P^\prime_{z_3,x}(y)$ from $P^\prime$, where
\[
P^\prime = \sum_{z_2} P(y|z_3,z_2,z_1,x)P(x|z_3,z_2,z_1)P(z_2|z_3)P(z_3),
\]
in the subgraph $G[\{Z_3,X,Y\}]$ shown in Figure~\ref{fig:recursive1}(\subref{fig:recursive1_sub}). In this graph $Z_3$ is connected to other vertices only through $X$ and it can be removed according to Corollary~\ref{cor:removal}, since the corresponding latent projection is the subgraph $G[\{X,Y\}]$. Thus we sum out $Z_3$ from $P^\prime$ and the resulting expression for the causal effect is
\[
\sum_{z_4} \frac{\sum_{z_2,z_3} P(y|z_3,z_2,z_1,x,z_4)P(z_4|z_3,z_2,z_1,x)P(x|z_3,z_2,z_1)P(z_2|z_3)P(z_3)}{\sum_{z_2,z_3,y^\prime} P(y^\prime|z_3,z_2,z_1,x,z_4)P(z_4|z_3,z_2,z_1,x)P(x|z_3,z_2,z_1)P(z_2|z_3)P(z_3)}.
\]
If Corollary~\ref{cor:removal} is not applied at this stage, the final expression is instead
\begin{align*}
& \sum_{z_4} \left(
  \frac{\sum_{z_2} P(y|z_3,z_2,z_1,x,z_4)P(z_4|z_3,z_2,z_1,x)P(x|z_3,z_2,z_1)P(z_2|z_3)P(z_3)}{\sum_{z_2,y^\prime} P(y^\prime|z_3,z_2,z_1,x,z_4)P(z_4|z_3,z_2,z_1,x)P(x|z_3,z_2,z_1)P(z_2|z_3)P(z_3)} \right. \times \\
  & \quad \left. \frac{\sum_{z_2,y^\prime} P(y^\prime|z_3,z_2,z_1,x,z_4)P(z_4|z_3,z_2,z_1,x)P(x|z_3,z_2,z_1)P(z_2|z_3)P(z_3)}{\sum_{z_2,y^\prime,z_4} P(y^\prime|z_3,z_2,z_1,x,z_4)P(z_4|z_3,z_2,z_1,x)P(x|z_3,z_2,z_1)P(z_2|z_3)P(z_3)} \right).
\end{align*}
Using Corollary~\ref{cor:removal} provides us with a simpler expression in this situation by completely removing the second term from the product inside the summation.

Next we show how Corollary~\ref{cor:removalsingle} can also be applied at a recursive step. Our interest lies in the causal effect of $X$ on $Y$ in graph $G$ of Figure~\ref{fig:recursive2}(\subref{fig:recursive2_start})

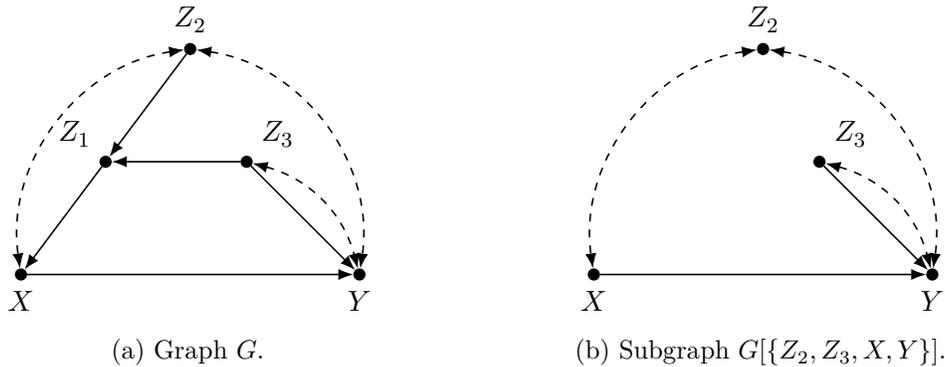
\begin{figure}[h]
  \begin{subfigure}[t]{0.5\textwidth}
  \centering
\begin{tikzpicture}[scale=1.5]
\node [dot = {0}{0}{Z_1}{above left}] at (0.75,1) {};
\node [dot = {0}{0}{Z_2}{above}] at (1.5,2) {};
\node [dot = {0}{0}{Z_3}{above right}] at (2,1) {};
\node [dot = {0}{0}{X}{below}] at (0,0) {};
\node [dot = {0}{0}{Y}{below}] at (3,0) {};

\draw [->] (Z_2) -- (Z_1);
\draw [->] (Z_1) -- (X);
\draw [->] (X)   -- (Y);
\draw [->] (Z_3) -- (Z_1);
\draw [->] (Z_3) -- (Y);

\draw [<->,dashed] (Z_2) to [bend right=45]  (X);
\draw [<->,dashed] (Z_2) to [bend left=45]   (Y);
\draw [<->,dashed] (Z_3) to [bend left=30]   (Y);
\end{tikzpicture}
  \caption{Graph $G$.}
  \label{fig:recursive2_start}
  \end{subfigure}
  \hspace{0.5cm}
  \begin{subfigure}[t]{0.4\textwidth}
  \centering
\begin{tikzpicture}[scale=1.5]
\node [dot = {0}{0}{Z_2}{above}] at (1.5,2) {};
\node [dot = {0}{0}{Z_3}{above right}] at (2,1) {};
\node [dot = {0}{0}{X}{below}] at (0,0) {};s
\node [dot = {0}{0}{Y}{below}] at (3,0) {};

\draw [->] (X)   -- (Y);
\draw [->] (Z_3) -- (Y);

\draw [<->,dashed] (Z_2) to [bend right=45]  (X);
\draw [<->,dashed] (Z_2) to [bend left=45]   (Y);
\draw [<->,dashed] (Z_3) to [bend left=30]   (Y);
\end{tikzpicture}
  \caption{Subgraph $G[\{Z_2,Z_3,X,Y\}]$.}
  \label{fig:recursive2_sub}
  \end{subfigure}
  \caption{A graph for the example of recursive application of Corollary~\ref{cor:removalsingle} within the ID algorithm.}
  \label{fig:recursive2}
\end{figure}
\noindent
There are no vertices that are connected to other vertices in the graph via a single vertex. When the ID algorithm is applied we eventually reach a step where the causal effect of $P^{\prime}_{z_2,x}(y)$ is to be identified from $P^\prime$, where
\[
  P^\prime = P(y|z_2,z_3,z_1,x)P(x|z_2,z_3,z_1)P(z_3|z_2)P(z_2),
\]
in the subgraph $G[\{Z_2,Z_3,X,Y\}]$ which is depicted in Figure~\ref{fig:recursive2}(\subref{fig:recursive2_sub}).

In this graph $Z_3$ can be removed according to Corollary~\ref{cor:removalsingle}, since $Z_3$ is connected to other vertices of the subgraph only through $Y$. The resulting distribution is obtained by summing out $Z_3$ from $P^\prime$. The final expression for the causal effect is now
\begin{equation} \label{eq:simplifiable}
  \frac{\sum_{z_2,z_3}P(y|z_2,z_3,z_1,x)P(x|z_2,z_3,z_1)P(z_3|z_2)P(z_2)}{\sum_{z_2,z_3,y^\prime}P(y^\prime|z_2,z_3,z_1,x)P(x|z_2,z_3,z_1)P(z_3|z_2)P(z_2)}.
\end{equation}
If Corollary~\ref{cor:removalsingle} is not applied, the resulting expression is instead
\begin{align*}
&\sum_{z_3} \left( \frac{\sum_{z_2}P(y|z_2,z_3,z_1,x)P(x|z_2,z_3,z_1)P(z_3|z_2)P(z_2)}{\sum_{z_2,y^\prime}P(y^\prime|z_2,z_3,z_1,x)P(x|z_2,z_3,z_1)P(z_3|z_2)P(z_2)} \times \right. \\
&\quad \left. \vphantom{\sum_z} \sum_{z_2,x,y^\prime} P(y^\prime|z_2,z_3,z_1,x)P(x|z_2,z_3,z_1)P(z_3|z_2)P(z_2) \right).
\end{align*}
As in the previous example, the benefit of applying Corollary~\ref{cor:removalsingle} is apparent. 

We can also take advantage of latent projections recursively via Corollary~\ref{cor:latent} as shown in the next example. Our interest lies in the causal effect of $X$ on $Y$ in graph $G$ of Figure~\ref{fig:recursive3}.

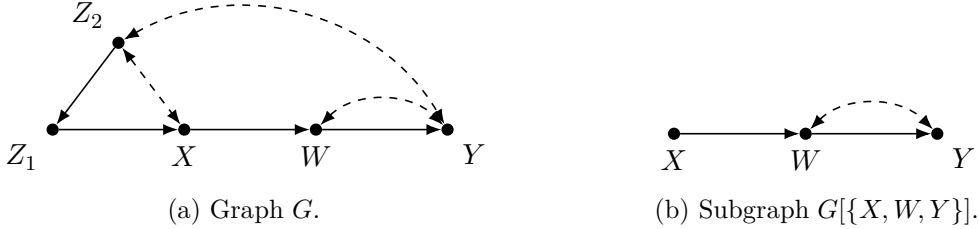
\begin{figure}[h]
  \begin{subfigure}[t]{0.5\textwidth}
  \centering
\begin{tikzpicture}[scale=1.75]
\node [dot = {0}{0}{Z_1}{below left}] at (-1,1) {};
\node [dot = {0}{0}{Z_2}{above left}] at (-0.5,1.66) {};
\node [dot = {0}{0}{W}{below}] at (1,1) {};
\node [dot = {0}{0}{X}{below}] at (0,1) {};
\node [dot = {0}{0}{Y}{below right}] at (2,1) {};

\draw [->] (Z_2) -- (Z_1);
\draw [->] (Z_1) -- (X);
\draw [->] (X)   -- (W);
\draw [->] (W)   -- (Y);

\draw [<->,dashed] (Z_2) to [bend right=0] (X);
\draw [<->,dashed] (Z_2) to [bend left=45]  (Y);
\draw [<->,dashed] (W)   to [bend left=45]  (Y);
\end{tikzpicture}
  \caption{Graph $G$.}
  \label{fig:recursive3_start}
  \end{subfigure}
  \hspace{0.5cm}
  \begin{subfigure}[t]{0.4\textwidth}
  \centering
\begin{tikzpicture}[scale=1.75]
\node [dot = {0}{0}{W}{below}] at (1,1) {};
\node [dot = {0}{0}{X}{below}] at (0,1) {};
\node [dot = {0}{0}{Y}{below right}] at (2,1) {};

\draw [->] (X)   -- (W);
\draw [->] (W)   -- (Y);

\draw [<->,dashed] (W)   to [bend left=45]  (Y);
\end{tikzpicture}
  \caption{Subgraph $G[\{X,W,Y\}]$.}
  \label{fig:recursive3_sub}
  \end{subfigure}
  \caption{A graph for the example of recursive application of Corollary~\ref{cor:latent} within the ID algorithm.}
  \label{fig:recursive3}
\end{figure}
\noindent
Here $X$ is connected to its child $W$ via a bidirected path and thus they belong to the same C-component rendering Corollary~\ref{cor:latent} unusable at this time. However, as a recursive step of the ID algorithm we have to identify $P^\prime_{x}(y)$ from $P^\prime$, where
\[
  P^\prime = \sum_{z_2} P(y|z_2,z_1,x,w)P(w|z_2,z_1,x)P(x|z_2,z_1)P(z_2)
\]
in the subgraph $G[\{X,W,Y\}]$. In this subgraph $X$ is not connected to its child $W$ via a bidirected path. We also find that the latent projection $L(G[\{X,W,Y\}], \{X,Y\})$ does not induce such a path between $X$ and $Y$. We can now continue identification in this latent projection and sum out $W$ from $P^\prime$. The resulting expression for the causal effect is
\[
  \frac{\sum_{z_2,w} P(y|z_2,z_1,x,w)P(w|z_2,z_1,x)P(x|z_2,z_1)P(z_2)}{\sum_{z_2,w, y^\prime} P(y^\prime|z_2,z_1,x,w)P(w|z_2,z_1,x)P(x|z_2,z_1)P(z_2)},
\]
whereas the expression without applying the corollary is instead
\begin{align*}
  & \sum_{w} \left( \frac{\sum_{z_2} P(y|z_2,z_1,x,w)P(w|z_2,z_1,x)P(x|z_2,z_1)P(z_2)}{\sum_{z_2,y^\prime} P(y^\prime|z_2,z_1,x,w)P(w|z_2,z_1,x)P(x|z_2,z_1)P(z_2)} \right. \times \\
  & \quad \frac{\sum_{z_2,y^\prime} P(y^\prime|z_2,z_1,x,w)P(w|z_2,z_1,x)P(x|z_2,z_1)P(z_2)}{\sum_{z_2,w^\prime,y^\prime} P(y^\prime|z_2,z_1,x,w)P(w^\prime|z_2,z_1,x)P(x|z_2,z_1)P(z_2)}.
\end{align*}
Additional examples are provided as an R script \citep{Rsoft} at the JMLR online paper repository. The script also includes all of the examples presented in this paper.

An interesting question is how pruning works together with simplification  presented in \citep{tikka17b}. We return to the example on identifying the causal effect of $X$ on $Y$ in the graph of Figure~\ref{fig:recursive2}. If we apply the ID algorithm without pruning and perform simplification as a post-processing step, then the resulting expression is
\begin{equation} \label{eq:simplifiable2}
\sum_{z_3}\frac{\sum_{z_2}P(y|z_2,z_3,z_1,x)P(x|z_2,z_3,z_1)P(z_3|z_2)P(z_2)}{\sum_{z_2}P(x|z_2,z_3,z_1)P(z_3|z_2)P(z_2)}P(z_3|z_2).
\end{equation}
This expression is in some aspects simpler than expression \eqref{eq:simplifiable} obtained using pruning alone, but does contain a sum over $Z_3$ that was not originally present.

When pruning is introduced to the ID algorithm and simplification is again applied, the resulting expression is instead
\[
\frac{\sum_{z_3,z_2}P(y|z_2,z_3,z_1,x)P(x|z_2,z_3,z_1)P(z_3|z_2)P(z_2)}{\sum_{z_2}P(x|z_2,z_1)P(z_2)},
\]
which is noticeably simpler than expressions \eqref{eq:simplifiable} and \eqref{eq:simplifiable2}. This example shows, that when pruning methods are employed together with simplification, a simpler expression can be reached than what is possible with either pruning or simplification alone.

\section{Discussion} \label{sect:disc}

We have presented criteria for removing variables from causal models that are not necessary to achieve identifiability for a given causal effect, and showed how these criteria can be applied in practice. We integrated our results into a new version of the ID algorithm called PID as presented in Algorithm~\ref{alg:aug_identify} to facilitate automatic processing of identifiability queries. It should be noted that the ID algorithm already performs some pruning such as removing non-ancestors of $\+ Y$.

The pruning operations carried out by Algorithm~\ref{alg:aug_identify} can significantly simplify the resulting expression compared to the traditional ID algorithm. Benefits of simplification can be realized in various settings. A simpler expression is easier to understand and to evaluate, since the dimensionality of the problem has been reduced. This is especially true for settings where the expression has to be evaluated repeatedly. Simplification can also help dealing with data where some variables are affected by bias or contain missing data. Obtaining an expression that does not contain these variables has clear advantages. It may also be of interest to obtain a different expression for the same causal effect.

The choice of variable ordering on line 5 of Algorithm~\ref{alg:aug_identify} is not arbitrary. However, available external knowledge may guide our selection to prefer certain orderings. For example, in a situation where two latent projections are mutually exclusive, we may prefer an ordering where the variable that is associated with the smallest cost, or of which we have the most accurate measurements is not considered latent.

When PID is applied in conjunction with simplification methods described in \citep{tikka17b}, various situations that lead to complex expressions can be taken into account. These methods complement each other, since the results in this paper deal with completely removing variables from the resulting expression, whereas the simplification methods focus on symbolic summation of so-called atomic expressions, which are expression consisting of a single sum and a number of product terms. An expression for a causal effect may consist of multiple atomic expressions, some of which can be simplified and some of which can not.

We showed via examples that our improvements are not simply pre-processing steps to be carried out before calling the ID algorithm, but actually provide significant benefits when applied recursively. As the ID algorithm manipulates the original graph it often enables the application of our results as well. As hedges characterize identifiability, it is possible to consider latent projections in a more general manner, but this is not necessarily beneficial for simplification. One could construct an algorithm that performs a search over the possible subsets of $\+ V$, and checks whether identifiability is retained in the corresponding latent projection. However, as we have shown via examples, this may not be enough to obtain a simpler expression, and the recursive structure of the ID algorithm needs to be taken advantage of. Instead, we could consider a variant of the PID algorithm, where line 5 is replaced by this procedure. However, one must be careful when applying this method, so that the computation does not become intractable when the number of vertices increases due to the complexity of the search.

\appendix

\acks{We wish to thank Professor Jukka Nyblom for his comments that greatly helped to improve this paper. We also thank the anonymous reviewers for their insightful feedback. The work belongs to the profiling area ''Decision analytics utilizing causal models and multiobjective optimization'' (DEMO) supported by Academy of Finland (grant number 311877).}

\vskip 0.2in
\bibliography{improvements}

\end{document}